%% file: main.tex
\documentclass{article} %
\usepackage{iclr2023_conference,times}

\input{headers/header-lqa-head}

\input{headers/header}

\input{headers/header-lqa-tail}

\title{Certified Training: \\Small Boxes are All You Need}
\author{Mark Niklas Müller\thanks{Equal contribution},~~Franziska Eckert\footnotemark[1],~~Marc Fischer \& Martin Vechev \\
	Department of Computer Science\\
	ETH Zurich, Switzerland\\
	\texttt{\{mark.mueller, marc.fischer, martin.vechev\}@inf.ethz.ch, eckertf@student.ethz.ch}
	\vspace*{-2mm}
}

\iclrfinalcopy %

\begin{document}

\maketitle

\input{abstract}

\input{introduction}

\input{background}

\input{method}

\input{theory}

\input{experiments}

\input{related_work}

\input{conclusion}

\input{statements}

\input{acknowledgements.tex}

\message{^^JLASTBODYPAGE \thepage^^J}

\clearpage
\bibliography{references}
\bibliographystyle{iclr2023_conference}

\message{^^JLASTREFERENCESPAGE \thepage^^J}

\ifincludeappendixx
	\clearpage
	\appendix
	\include{appendix}
\fi

\end{document}

%% file: headers/header-lqa-head.tex
\usepackage[utf8]{inputenc} %

\usepackage{graphicx}

\usepackage[T1]{fontenc}

\input{headers/lqa/overfull}

\input{headers/lqa/comments}

\input{headers/lqa/appendix}

\input{headers/lqa/colors}

\input{headers/lqa/listing}

%% file: headers/lqa/overfull.tex
\IfFileExists{headers/config/showoverfull.config}{
	\overfullrule=1cm
}{
}

%% file: headers/lqa/comments.tex
\usepackage{marginnote}

\usepackage[backgroundcolor=none,linecolor=red,textsize=footnotesize]{todonotes}

%% file: headers/lqa/appendix.tex
\usepackage{etoolbox}

\newbool{includeappendix}
\setbool{includeappendix}{true} %
\IfFileExists{headers/config/noappendix.config}{
	\setbool{includeappendix}{false}
}{}

\newif\ifincludeappendixx
\ifbool{includeappendix}{
	\includeappendixxtrue
}{
	\includeappendixxfalse
}

\usepackage{xr} %
\usepackage{filecontents}

\ifbool{includeappendix}{}{
	\input{appendix-labels-loader}

	\externaldocument{appendix-labels}
}

%% file: headers/lqa/colors.tex
\definecolor{my-full-blue}{HTML}{1F77B4}

\definecolor{my-full-orange}{HTML}{FF7F0E}

\definecolor{my-full-green}{HTML}{2CA02C}

\definecolor{my-full-red}{HTML}{d62728}

\definecolor{my-full-purple}{HTML}{9467bd}

\definecolor{my-full-brown}{HTML}{8c564b}

\definecolor{my-full-pink}{HTML}{e377c2}

\definecolor{my-full-gray}{HTML}{7f7f7f}

\definecolor{my-full-olive}{HTML}{bcbd22}

\definecolor{my-full-cyan}{HTML}{17becf}

\definecolor{netborder}{RGB}{  150, 109, 80}
\definecolor{netinside}{RGB}{ 204, 204, 204}

\definecolor{Sborder}{RGB}{31,120,180}
\definecolor{Sinside}{RGB}{166,206,227}

\definecolor{Pinside}{RGB}{230, 150, 71}
\definecolor{Pborder}{RGB}{186, 109, 34}

\definecolor{DPinside}{RGB}{151, 155, 156}

\definecolor{Correct}{RGB}{148, 206, 103}
\definecolor{Wrong}{RGB}{232, 116, 116}

\definecolor{ibpm}{RGB}{0, 0, 0}
\definecolor{sbbm}{RGB}{0, 0, 0}

\definecolor{cpgd}{RGB}{205, 205, 205}
\definecolor{cibp}{RGB}{152, 68, 100}
\definecolor{csbb}{RGB}{94, 204, 171}
\definecolor{cibpr}{RGB}{192, 175, 251}
\definecolor{ccolt}{RGB}{230, 161, 118}
\definecolor{ccibp}{RGB}{0, 103, 138}

\colorlet{my-blue}{my-full-blue!30}
\colorlet{my-orange}{my-full-orange!30}
\colorlet{my-green}{my-full-green!30}
\colorlet{my-red}{my-full-red!30}
\colorlet{my-purple}{my-full-purple!30}
\colorlet{my-brown}{my-full-brown!30}
\colorlet{my-pink}{my-full-pink!30}
\colorlet{my-gray}{my-full-gray!30}
\colorlet{my-olive}{my-full-olive!30}
\colorlet{my-cyan}{my-full-cyan!30}

%% file: headers/lqa/listing.tex
\usepackage{listings}

\usepackage{textcomp}

\usepackage[scaled=0.8]{beramono}

\definecolor{ckeyword}{HTML}{7F0055}
\definecolor{ccomment}{HTML}{3F7F5F}
\definecolor{cstring}{HTML}{2A0099}

\lstdefinestyle{numbers}{
	numbers=left,
	framexleftmargin=20pt,
	numberstyle=\tiny,
	firstnumber=auto,
	numbersep=1em,
	xleftmargin=2em
}

\lstdefinestyle{layout}{
	frame=none,
	captionpos=b,
}

\lstdefinestyle{comment-style}{
	morecomment=[l]//,
	morecomment=[s]{/*}{*/},
	commentstyle={\color{ccomment}\itshape},
}

\lstdefinestyle{string-style}{
	morestring=[b]",%
	morestring=[b]',%
	stringstyle={\color{cstring}},
	showstringspaces=false,%
}

\lstdefinestyle{keyword-style}{
	keywordstyle={\ttfamily\bfseries},
	morekeywords={
		function,
		constructor,
		int,
		bool,
		return,
		returns,
		uint
	},
	morekeywords = [2]{},
	keywordstyle = [2]{\text},
	sensitive=true,
}

\lstdefinestyle{input-encoding}{
	inputencoding=utf8,
	extendedchars=true,
	literate=
	{ℝ}{$\reals$}1%
	{→}{$\rightarrow$}1%
	{α}{$\alpha$}1%
	{β}{$\beta$}1%
	{λ}{$\lambda$}1%
	{θ}{$\theta$}1%
	{ϕ}{$\phi$}1%
}

\lstdefinestyle{escaping}{
	moredelim={**[is][\color{blue}]{\%}{\%}},
	escapechar=|,
	mathescape=true
}

\lstdefinestyle{default-style}{
	basicstyle=\fontencoding{T1}\ttfamily\footnotesize,
	style=numbers,
	style=layout,
	style=comment-style,
	style=string-style,
	style=keyword-style,
	style=input-encoding,
	style=escaping,
	tabsize=2,
	upquote=true
}

\lstdefinelanguage{BASIC}{
	language=C++,
	style=default-style
}[keywords,comments,strings]%

%% file: headers/header.tex
\input{math_commands.tex}

\usepackage[hidelinks]{hyperref}
\usepackage{url}

\usepackage{adjustbox}
\usepackage[capitalise]{cleveref}

\usepackage{enumitem}
\usepackage{threeparttable}
\usepackage{multirow, makecell}
\usepackage{booktabs}
\usepackage{xspace}
\usepackage{wrapfig}

\usepackage{subcaption}
\usepackage{caption}
\usepackage{algorithm,algcompatible,amssymb,amsmath}

\renewcommand{\COMMENT}[2][.5\linewidth]{\leavevmode\hfill\makebox[#1][l]{//~#2}}
\algnewcommand\RETURN{\State \textbf{return} }
\usepackage{amsthm}
\usepackage{thmtools}
\usepackage{thm-restate}

\crefname{thm}{Theorem}{Theorems}

\crefname{lem}{Lemma}{Lemmas}

\usepackage{tikz}
\usetikzlibrary{calc,decorations,decorations.pathmorphing}
\usetikzlibrary{positioning,fit,arrows}
\usetikzlibrary{decorations.markings}
\usetikzlibrary{shapes,shapes.geometric}
\usetikzlibrary{shadows,patterns,snakes}
\usetikzlibrary{backgrounds,decorations.pathreplacing,automata}
\usetikzlibrary{intersections}
\usetikzlibrary{angles,quotes}
\usetikzlibrary{plotmarks}
\usetikzlibrary{patterns}
\usetikzlibrary{arrows.meta}
\usetikzlibrary{shapes.misc}
\usetikzlibrary{chains}
\usepackage{siunitx}
\newcolumntype{d}[1]{S[table-format=#1]}

\makeatletter
\def\extractcoord#1#2#3{
	\path let \p1=(#3) in \pgfextra{
		\pgfmathsetmacro#1{\x{1}/\pgf@xx}
		\pgfmathsetmacro#2{\y{1}/\pgf@yy}
		\xdef#1{#1} \xdef#2{#2}
	};
}
\makeatother

\newcommand{\bc}[1]{\mathcal{#1}}
\newcommand{\bs}[1]{\boldsymbol{#1}}
\DeclareMathOperator*{\relu}{ReLU}

\newcommand{\tool}{\textsc{SABR}\xspace}
\newcommand{\tooll}{\textbf{S}mall \textbf{A}dversarial \textbf{B}ounding \textbf{R}egions\xspace}

\newcommand{\ratiol}{subselection ratio\xspace}
\newcommand{\ratiols}{subselection ratios\xspace}
\newcommand{\ratios}{\lambda \xspace}

\newcommand{\mnbab}{\textsc{MN-BaB}\xspace}
\newcommand{\milp}{\textsc{MILP}\xspace}
\newcommand{\bcrown}{$\beta$-\textsc{CROWN}\xspace}
\newcommand{\ibp}{\textsc{IBP}\xspace}

\newcommand{\crownibp}{\textsc{CROWN-IBP}\xspace}
\newcommand{\deeppoly}{\textsc{DeepPoly}\xspace}
\newcommand{\deepz}{\textsc{DeepZ}\xspace}
\newcommand{\ibpr}{\textsc{IBP-R}\xspace}
\newcommand{\diffai}{\textsc{DiffAI}\xspace}
\newcommand{\colt}{\textsc{COLT}\xspace}
\newcommand{\fgsm}{\textsc{FGSM}\xspace}
\newcommand{\boxd}{\textsc{Box}\xspace}

\newcommand{\sortnet}{\textsc{SortNet}\xspace}

\newcommand{\cifar}{CIFAR-10\xspace}
\newcommand{\TIN}{\textsc{TinyImageNet}\xspace}
\newcommand{\mnist}{\textsc{MNIST}\xspace}

\newcommand{\cnns}{\texttt{CNN7}\xspace}
\newcommand{\cnnsn}{\texttt{CNN7-narrow}\xspace}

\renewcommand{\th}{\textsuperscript{th}\xspace}

\newcommand{\barrows}{\tikz[]{\draw[-{Triangle[width=2.2pt,length=1.5pt]}, line width=1.0pt,Sborder](0,0.22) -- (0, 0);}\xspace}
\newcommand{\barrowm}{\tikz[]{\draw[-{Triangle[width=3pt,length=2pt]}, line width=1.5pt,Sborder](0,0.22) -- (0, 0);}\xspace}
\newcommand{\barrowl}{\tikz[]{\draw[-{Triangle[width=4pt,length=3pt]}, line width=2.0pt,Sborder](0,0.22) -- (0, 0);}\xspace}
\newcommand{\tbox}[1]{\tikz[]{\def \x{#1} \draw[-, line width=0.6pt](0,0) -- (0, \x) -- (\x, \x) -- (\x, 0) --cycle;}\xspace}
\newcommand{\tdbox}[1]{\tikz[]{\def \x{#1} \draw[-, line width=0.6pt, dash pattern=on 1.5pt off 1pt](0,0) -- (0, \x) -- (\x, \x) -- (\x, 0) --cycle;}\xspace}

\newcommand{\markersbb}{\tikz[]{\node[fill, diamond, aspect=1, color=csbb, inner sep=0pt, minimum size=2.5mm]{};}\xspace}
\newcommand{\markeribpr}{\tikz[]{\node[fill, circle, aspect=1, color=cibpr, inner sep=0pt, minimum size=2.5mm]{};}\xspace}
\newcommand{\markerccibp}{\tikz[]{\node[fill, aspect=1, color=ccibp, inner sep=0pt, minimum size=2.1mm]{};}\xspace}
\newcommand{\markeribp}{\tikz[]{\def \x{0.19} \draw[-, line width=1.8pt,cibp ](0,0) -- (0, \x)	(-0.5*\x, 0.5*\x) -- (0.5*\x,  0.5*\x);}\xspace}
\newcommand{\markercolt}{\tikz[]{\def \x{0.18} \draw[-, line width=1.8pt, ccolt](0,0) -- (\x, \x) (0, \x) -- (\x, 0);}\xspace}

\newcommand{\markerb}[1]{\tikz[]{\node[fill, aspect=1, color=#1, inner sep=0pt, minimum size=2.1mm]{};}\xspace}

%% file: math_commands.tex
\usepackage{amsmath,amsfonts,bm}

\def\eqref#1{equation~\ref{#1}}

\def\1{\bm{1}}

\def\eps{{\epsilon}}

\def\vb{{\bm{b}}}

\def\vf{{\bm{f}}}

\def\vh{{\bm{h}}}

\def\vl{{\bm{l}}}

\def\vu{{\bm{u}}}

\def\vx{{\bm{x}}}
\def\vy{{\bm{y}}}

\def\mT{{\bm{T}}}

\def\mW{{\bm{W}}}
\def\mX{{\bm{X}}}

\DeclareMathAlphabet{\mathsfit}{\encodingdefault}{\sfdefault}{m}{sl}
\SetMathAlphabet{\mathsfit}{bold}{\encodingdefault}{\sfdefault}{bx}{n}

\newcommand{\E}{\mathbb{E}}

\newcommand{\R}{\mathbb{R}}

\DeclareMathOperator*{\argmax}{arg\,max}
\DeclareMathOperator*{\argmin}{arg\,min}

\DeclareMathOperator{\sign}{sign}

%% file: headers/header-lqa-tail.tex
\input{headers/lqa/references}

%% file: headers/lqa/references.tex
\makeatletter
\renewcommand\theHALG@line{\thealgorithm.\arabic{ALG@line}}
\makeatother

\crefformat{section}{\S#2#1#3}

\crefrangeformat{section}{\S#3#1#4\crefrangeconjunction\S#5#2#6}

\crefmultiformat{section}{\S#2#1#3}{\crefpairconjunction\S#2#1#3}{\crefmiddleconjunction\S#2#1#3}{\creflastconjunction\S#2#1#3}

\newcommand{\crefrangeconjunction}{--}

\crefname{listing}{Lst.}{listings}
\crefname{line}{Lin.}{Lin.}
\crefname{appendix}{App.}{App.}

\newcommand{\appref}[1]{%
	\ifbool{includeappendix}{\cref{#1}}{the appendix}%
}
\newcommand{\Appref}[1]{%
	\ifbool{includeappendix}{\cref{#1}}{The appendix}%
}

%% file: abstract.tex
\begin{abstract}
To obtain, deterministic guarantees of adversarial robustness, specialized training methods are used. We propose, \tool, a novel such certified training method, based on the key insight that propagating interval bounds for a small but carefully selected subset of the adversarial input region is sufficient to approximate the worst-case loss over the whole region while significantly reducing approximation errors. We show in an extensive empirical evaluation that \tool outperforms existing certified defenses in terms of both \emph{standard and certifiable accuracies} across perturbation magnitudes and datasets, pointing to a new class of certified training methods promising to alleviate the robustness-accuracy trade-off.
\end{abstract}

%% file: introduction.tex
\section{Introduction} \label{sec:introduction}
As neural networks are increasingly deployed in safety-critical domains, formal robustness guarantees against adversarial examples \citep{BiggioCMNSLGR13,SzegedyZSBEGF13} are becoming ever more important.
However, despite significant progress, specialized training methods that improve certifiability at the cost of severely reduced accuracies are still required to obtain deterministic guarantees.

Given an input region defined by an adversary specification, both training and certification methods compute a network's reachable set by propagating a symbolic over-approximation of this region through the network \citep{SinghGMPV18,SinghGPV19,GowalDSBQUAMK18}. Depending on the propagation method, both the computational complexity and approximation-tightness can vary widely. For certified training, an over-approximation of the worst-case loss is computed from this reachable set and then optimized \citep{MirmanGV18,WongSMK18}. Surprisingly, the least precise propagation methods yield the highest certified accuracies as more precise methods induce harder optimization problems \citep{JovanovicBBV21}. However, the large approximation errors incurred by these imprecise methods lead to over-regularization and thus poor accuracy. Combining precise worst-case loss approximations and a tractable optimization problem is thus the core challenge of certified training.

In this work, we tackle this challenge and propose a novel certified training method, \tool, \tooll, based on the following key insight: by propagating small but carefully selected subsets of the adversarial input region with imprecise methods (i.e., \boxd), we can obtain \emph{both} well-behaved optimization problems and precise approximations of the worst-case loss. 
This yields less over-regularized networks, allowing \tool to improve on state-of-the-art certified defenses in terms of both standard \emph{and} certified accuracies across settings, thereby pointing to a new class of certified training methods.

\vspace{-1mm}
\paragraph{Main Contributions} Our main contributions are:
\vspace{-1.5mm}
\begin{itemize}
 \setlength\itemsep{0.15mm}
\item A novel certified training method, \tool, reducing over-regularization to improve both standard and certified accuracy (\cref{sec:method}).
\item A theoretical investigation motivating \tool by deriving new insights into the growth of \boxd relaxations during propagation (\cref{sec:theory}).
\item An extensive empirical evaluation demonstrating that \tool outperforms \emph{all} state-of-the-art certified training methods in terms of both \emph{standard and certifiable accuracies} on \mnist, \cifar, and \TIN (\cref{sec:experiments}).
\end{itemize}

%% file: background.tex
\section{Background} \label{sec:background}
In this section, we provide the necessary background for \tool.
\paragraph{Adversarial Robustness}
Consider a classification model $\vh\colon \R^{d_\text{in}} \mapsto \R^{c}$ that, given an input $\vx \in \bc{X} \subseteq \R^{d_\text{in}}$, predicts numerical scores $\vy := \vh(\vx)$ for every class. We say that $\vh$ is adversarially robust on an $\ell_p$-norm ball $\bc{B}_p^{\epsilon_p}(\vx)$ of radius $\epsilon_p$ if it consistently predicts the target class $t$ for all perturbed inputs $\vx' \in \bc{B}_p^{\epsilon_p}(\vx)$. More formally, we define \emph{adversarial robustness} as:
\begin{equation}
	\label{eq:adv_robustness}
	\argmax_j h(\vx')_j = t, \quad \forall \vx' \in \bc{B}_p^{\epsilon_p}(\vx) := \{\vx' \in \bc{X} \mid \|\vx -\vx'\|_p \leq \epsilon_p\}.
\end{equation}
\paragraph{Neural Network Verification}
\input{./figures/ibp_figure.tex}

To verify that a neural network $\vh$ is adversarially robust, several verification techniques have been proposed.

A simple but effective such method is verification with the \boxd relaxation \citep{MirmanGV18}, also called interval bound propagation (IBP) \citep{GowalIBP2018}. Conceptually, we first compute an over-approximation of a network's reachable set by propagating the input region $\bc{B}_p^{\epsilon_p}(\vx)$ through the neural network and then check whether all outputs in the reachable set yield the correct classification.
This propagation sequentially computes a hyper-box (each dimension is described as an interval) relaxation of a layer's output, given a hyper-box input.
As an example, consider an $L$-layer network $\vh = \vf_L \circ \bs{\sigma} \circ \vf_{L-2} \circ \dotsc \circ \vf_1$, with linear layers $\vf_i$ and ReLU activation functions $\bs{\sigma}$.
Given an input region $\bc{B}_p^{\epsilon_p}(\vx)$, we over-approximate it as a hyper-box, centered at $\bar{\vx}^0 := \vx$ and with radius $\bs{\delta}^0 := \epsilon_p$, such that we have the $i$\th dimension of the input $\vx^0_i \in [\bar{x}^0_i - \delta^0_i, \bar{x}^0_i + \delta^0_i]$.
Given a linear layer $\vf_i(\vx^{i-1}) = \mW \vx^{i-1} + \vb =: \vx^i$, we obtain the hyper-box relaxation of its output with centre $\bar{\vx}^i = \mW \bar{\vx}^{i-1} + \vb$ and radius $\bs{\delta}^{i} = |\mW| \bs{\delta}^{i-1}$, where $| \cdot |$ denotes the elementwise absolute value.
A ReLU activation $\relu(\vx^{i-1}) := \max(0,\vx^{i-1})$ can be over-approximated by propagating the lower and upper bound separately, resulting in a output hyper-box with $\bar{\vx}^{i} = \tfrac{\vu^i + \vl^i}{2}$ and $\bs{\delta}^i = \tfrac{\vu^i - \vl^i}{2}$
where  $\vl^i = \relu(\bar{\vx}^{i-1} - \bs{\delta}^{i-1})$ and $\vu^i = \relu(\bar{\vx}^{i-1} + \bs{\delta}^{i-1})$.
Proceeding this way for all layers, we obtain lower and upper bounds on the network output $\vy$ and can check if the output score of the target class is greater than that of all other classes by computing the upper bound on the logit difference  $y^\Delta_i := y_i - y_t$ and then checking whether $y^\Delta_i < 0, \; \forall i \neq t$.

We illustrate this propagation process for a one-layer network in \cref{fig:ibp_example}. There, the blue shapes (\markerb{blue!40}) show an exact propagation of the input region and the red shapes (\markerb{red!40}) their hyper-box relaxation. Note how after the first linear and ReLU layer (third row), the relaxation (red) contains already many points not reachable via exact propagation (blue), despite it being the smallest hyper-box containing the exact region. These so-called approximation errors accumulate quickly, leading to an increasingly imprecise abstraction, as can be seen by comparing the two shapes after an additional linear layer (last row).
To verify that this network classifies all inputs in $[-1,1]^2$ to class $1$, we have to show the upper bound of the logit difference $y_2-y_1$ to be less than $0$. While the concrete maximum of $-0.3 \geq y_2-y_1$ (black $\times$) is indeed less than $0$, showing that the network is robust, the \boxd relaxation only yields $0.6 \geq y_2-y_1$ (red {\color{red}$\times$}) and is thus too imprecise to prove it.

Beyond \boxd, more precise verification approaches track more relational information at the cost of increased computational complexity \citep{PalmaIBPR22,WangZXLJHK21}.
A recent example is \mnbab \citep{FerrariMJV22}, which improves on \boxd in two key ways: First, instead of propagating axis-aligned hyper-boxes, it uses much more expressive polyhedra, allowing linear layers to be captured exactly and ReLU layers much more precisely. Second, if the result is still too imprecise, the verification problem is recursively split into easier ones, by introducing a case distinction between the two linear segments of the ReLU function. This is called the branch-and-bound (BaB) approach \citep{BunelLTTKK20}. We refer the interested reader to \citet{FerrariMJV22} for more details.

\paragraph{Training for Robustness}
For neural networks to be certifiably robust, special training is necessary.
Given a data distribution $(\vx, t) \sim \bc{D}$, standard training generally aims to find a network parametrization $\bs{\theta}$ that minimizes the expected cross-entropy loss (see \cref{app:ce_deviation}):
\begin{equation}\label{eq:std_train}
	\theta_\text{std} = \argmin_\theta \E_\bc{D} [\bc{L}_\text{CE}(\vh_{\bs{\theta}}(\vx),t)], \quad \text{with} \quad \bc{L}_\text{CE}(\vy, t) = \ln\big(1 + \sum_{i \neq t} \exp(y_i-y_t)\big).
\end{equation}
When training for robustness, we, instead, wish to minimize the expected \emph{worst-case loss} around the data distribution, leading to the min-max optimization problem:
\begin{equation}
\label{eq:rob_opt}
	\theta_\text{rob} = \argmin_\theta \mathbb{E}_{\bc{D}} \big[ \max_{\vx' \in \bc{B}_p^{\epsilon_p}(\vx) }\bc{L}_\text{CE}(\vh_{\bs{\theta}}(\vx'),t) \big].
\end{equation}
Unfortunately, solving the inner maximization problem is generally intractable. Therefore, it is commonly under- or over-approximated, yielding adversarial and certified training, respectively. For notational clarity, we henceforth drop the subscript $p$.

\paragraph{Adversarial Training}
Adversarial training optimizes a lower bound on the inner optimization objective in \cref{eq:rob_opt} by first computing concrete examples $\vx'\in \bc{B}^{\epsilon}(\vx)$ maximizing the loss term and then optimizing the network parameters $\bs{\theta}$ for these samples.
Typically, $\vx'$ is computed by initializing $\vx'_0$ uniformly at random in $\bc{B}^{\epsilon}(\vx)$ and then updating it over $N$ projected gradient descent steps (PGD) \citep{MadryMSTV18} 
$\vx'_{n+1}=\Pi_{\bc{B}^{\epsilon}(\vx)}\vx'_n + \alpha \sign(\nabla_{\vx'_n} \bc{L}_\text{CE}(\vh_{\bs{\theta}}(\vx'_n),t))$, 
with step size $\alpha$ and projection operator $\Pi$.
While networks trained this way typically exhibit good empirical robustness, they remain hard to formally verify and sometimes vulnerable to stronger or different attacks \citep{TramerCBM20,Croce020a}.

\paragraph{Certified Training} Certified training optimizes an upper bound on the inner maximization objective in \cref{eq:rob_opt}, obtained via a bound propagation method. These methods compute an upper bound $\vu_{\vy^\Delta}$ on the logit differences $\vy^\Delta := \vy - y_t$, as described above, to obtain the robust cross-entropy loss $\bc{L}_\text{CE,rob}(\bc{B}^{\epsilon}(\vx),t) = \bc{L}_\text{CE}(\vu_{\vy^\Delta},t)$.
We will use \boxd to refer to the verification and propagation approach, and \ibp to refer to the corresponding training method.

Surprisingly, using the imprecise \boxd relaxation \citep{MirmanGV18,GowalIBP2018,ShiBetterInit2021} consistently produces better results than methods based on tighter abstractions \citep{ZhangCXGSLBH20,BalunovicV20,WongSMK18}. \citet{JovanovicBBV21} trace this back to the optimization problems induced by the more precise methods becoming intractable to solve. 
While the heavily regularized, \ibp trained networks are amenable to certification, they suffer from severely reduced (standard) accuracies. Overcoming this robustness-accuracy trade-off remains a key challenge of robust machine learning.

%% file: figures/ibp_figure.tex
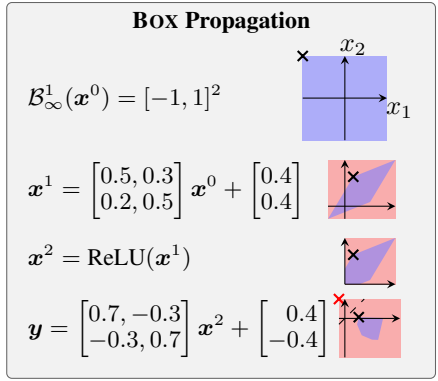
\begin{wrapfigure}[18]{r}{0.417 \textwidth}
\vspace{-4mm}
\centering
\tikzset{
	point/.style={
		thick,
		draw=black,
		cross out,
		inner sep=0pt,
		minimum width=3pt,
		minimum height=3pt,
	},
	point_r/.style={
		thick,
		draw=red,
		cross out,
		inner sep=0pt,
		minimum width=3pt,
		minimum height=3pt,
	},
}

\begin{tikzpicture}[scale=0.70]
\tikzset{>=latex}

    \node[draw=black!60, fill=black!05, rectangle,     	rounded corners=2pt,
minimum width=5.7cm, minimum height=5.0cm] (bounding_box) at (0, 0) {
};
			\node[anchor=north] at (bounding_box.north) {\footnotesize\textbf{\boxd Propagation}};

	\def \dy{-1.35cm}
	\def \dx{-5.7}
	\def \ddx{1.85cm}

    \begin{scope}[xshift=\ddx, yshift=-1.3*\dy]
		\begin{scope}[xshift=0.5cm,scale=0.8]
			\coordinate (O) at ({0.0000},{0.0000});
			\coordinate (X) at ({1.0000},{0.0000});
			\coordinate (Y) at ({0.0000},{1.0000});
			\coordinate (XO) at ({-1.0000},{0.0000});
			\coordinate (YO) at ({0.0000},{-1.0000});

			\coordinate (input_head) at ({0.0000},{0.0000});
			\coordinate (input_p_0) at ({-1.0000},{-1.0000});
			\coordinate (input_p_1) at ({1.0000},{-1.0000});
			\coordinate (input_p_2) at ({1.0000},{1.0000});
			\coordinate (input_p_3) at ({-1.0000},{1.0000});
	
			\fill [fill=blue!50, opacity=0.6, rounded corners=0mm] (input_p_0) -- (input_p_1) -- (input_p_2) -- (input_p_3)  -- cycle;
			\draw[-stealth] (XO) -- (X);
			\draw[-stealth] (YO) -- (Y);
			
			\node at ($(X)+(0.3,-0.3)$) {$x_1$};
			\node at ($(Y)+(0.2,0.2)$) {$x_2$};
			\node[point] at (input_p_3) {};
	
			\end{scope}
		
			\node[anchor=west] at (\dx, 0) {\footnotesize $\bc{B}_\infty^{1}(\vx^0) = [-1,1]^2$};

		\end{scope}
		
		\begin{scope}[xshift=\ddx, yshift=-0*\dy]
			\begin{scope}[xshift=0.5cm,  yshift=-0.3cm,scale=0.8]%
			\coordinate (O) at ({0.0000},{0.0000});
			\coordinate (X) at ({1.2000},{0.0000});
			\coordinate (Y) at ({0.0000},{1.1000});
			\coordinate (XO) at ({-0.4000},{0.0000});
			\coordinate (YO) at ({0.0000},{-0.3000});
			
			\coordinate (linear_box_p_0) at ({-0.4000},{-0.3000});
			\coordinate (linear_box_p_1) at ({1.2000},{-0.3000});
			\coordinate (linear_box_p_2) at ({1.2000},{1.1000});
			\coordinate (linear_box_p_3) at ({-0.4000},{1.1000});
			\fill [fill=red!50, opacity=0.6, rounded corners=0mm] (linear_box_p_0) -- (linear_box_p_1) -- (linear_box_p_2) -- (linear_box_p_3)  -- cycle;
			
			\coordinate (linear_head) at ({0.4000},{0.4000});
			\coordinate (linear_p_0) at ({0.2000},{0.7000});
			\coordinate (linear_p_1) at ({-0.4000},{-0.3000});
			\coordinate (linear_p_2) at ({0.6000},{0.1000});
			\coordinate (linear_p_3) at ({1.2000},{1.1000});
			
			\fill [fill=blue!50, opacity=0.6, rounded corners=0mm] (linear_p_0) -- (linear_p_1) -- (linear_p_2) -- (linear_p_3)  -- cycle;
			
			\node[point] at (linear_p_0) {};
			
			\draw[-stealth] (XO) -- (X);
			\draw[-stealth] (YO) -- (Y);
		\end{scope}
		\node[anchor=west] at (\dx,0) {\footnotesize $\vx^1 = \begin{bmatrix}
			0.5, 0.3 \\
			0.2, 0.5 \\
			\end{bmatrix}\vx^0 + \begin{bmatrix}
			0.4\\0.4
			\end{bmatrix}$};
		
	\end{scope}
	
	\begin{scope}[xshift=\ddx, yshift=0.95*\dy]
		\begin{scope}[xshift=0.5cm,  yshift=-0.5cm,scale=0.8]%
		\coordinate (O) at ({0.0000},{0.0000});
		\coordinate (X) at ({1.2000},{0.0000});
		\coordinate (Y) at ({0.0000},{1.1000});
		\coordinate (XO) at ({-0.0000},{0.0000});
		\coordinate (YO) at ({0.0000},{-0.0000});
		\coordinate (XX) at ({5.0000},{0.0000});
		\coordinate (YY) at ({0.0000},{5.0000});
		
		\coordinate (linear_box_p_0) at ({-0.000},{-0.000});
		\coordinate (linear_box_p_1) at ({1.2000},{-0.000});
		\coordinate (linear_box_p_2) at ({1.2000},{1.1000});
		\coordinate (linear_box_p_3) at ({-0.000},{1.1000});
		\fill [fill=red!50, opacity=0.6, rounded corners=0mm] (linear_box_p_0) -- (linear_box_p_1) -- (linear_box_p_2) -- (linear_box_p_3)  -- cycle;

		\coordinate (linear_p_0) at ({0.2000},{0.7000});
		\coordinate (linear_p_1) at ({-0.4000},{-0.3000});
		\coordinate (linear_p_2) at ({0.6000},{0.1000});
		\coordinate (linear_p_3) at ({1.2000},{1.1000});

  		\path[name path=x] (O) -- (XX);
  		\path[name path=y] (O) -- (YY);
  		\path[name path=top] (linear_p_1) -- (linear_p_0);
  		\path[name path=bottom] (linear_p_1) -- (linear_p_2);
		\path[name intersections={of=y and top}];
		\coordinate (linear_p_4) at (intersection-1);
		\path[name intersections={of=x and bottom}];
		\coordinate (linear_p_5) at (intersection-1);

		\fill [fill=blue!50, opacity=0.6, rounded corners=0mm] (linear_p_0) -- (linear_p_4) --(O) -- (linear_p_5) -- (linear_p_2) -- (linear_p_3)  -- cycle;
		
		\node[point] at (linear_p_0) {};
				
		\draw[-stealth] (XO) -- (X);
		\draw[-stealth] (YO) -- (Y);
		\end{scope}
		\node[anchor=west] at (\dx, 0) {\footnotesize$\vx^2 = \text{ReLU}(\vx^1)$};
	
	\end{scope}

	\begin{scope}[xshift=\ddx, yshift=1.95*\dy]
		\begin{scope}[xshift=0.5cm, yshift=0.2cm,scale=0.8]%
			\coordinate (O) at ({0.0000},{0.0000});
			\coordinate (DA) at ({-0.1400},{-0.1400});
			\coordinate (DB) at ({0.4600},{0.4600});
			\coordinate (X) at ({1.3107},{0.0000});
			\coordinate (Y) at ({0.0000},{0.4600});
			\coordinate (XO) at ({-0.1400},{0.0000});
			\coordinate (YO) at ({0.0000},{-0.9250});					
			\coordinate (XX) at ({5.0000},{0.0000});
			\coordinate (YY) at ({0.0000},{5.0000});
			
			\coordinate (box_2_p_0) at ({-0.1400},{0.4600});
			\coordinate (box_2_p_1) at ({-0.1400},{-0.9250});
			\coordinate (box_2_p_2) at ({1.3107},{-0.9250});
			\coordinate (box_2_p_3) at ({1.3107},{0.4600});
			\fill [fill=red!50, opacity=0.6, rounded corners=0mm] (box_2_p_0) -- (box_2_p_1) -- (box_2_p_2) -- (box_2_p_3)  -- cycle;

			\coordinate (linear_2_p_0) at ({0.7900},{-0.5100});
			\coordinate (linear_2_p_1) at ({0.9100},{0.0100});
			\coordinate (linear_2_p_2) at ({0.3300},{0.0300});
			\coordinate (linear_2_p_3) at ({0.2900},{-0.1433});
			\coordinate (linear_2_p_4) at ({0.4000},{-0.4000});
			\coordinate (linear_2_p_5) at ({0.6450},{-0.5050});

			\fill [fill=blue!50, opacity=0.6, rounded corners=0mm] (linear_2_p_0) -- (linear_2_p_1) -- (linear_2_p_2) -- (linear_2_p_3) -- (linear_2_p_4) -- (linear_2_p_5)  -- cycle;

			\node[point] at (linear_2_p_2) {};
			\node[point_r] at (box_2_p_0) {};

			\draw[-stealth] (XO) -- (X);
			\draw[-stealth] (YO) -- (Y);
			\draw[-,dashed] (DA) -- (DB);
		\end{scope}

		\node[anchor=west] at (\dx,0) {\footnotesize $\vy = \begin{bmatrix}
			0.7, -0.3 \\
			-0.3, 0.7 \\
			\end{bmatrix}\vx^2 + \begin{bmatrix}
			\;\;\, 0.4\\-0.4
			\end{bmatrix}$};

\end{scope}

	\coordinate (z) at ({0.1231},{0.0846});
\end{tikzpicture}
\vspace{-6mm}
\caption{Comparison of exact (blue) and \boxd (red) propagation through a one layer network. We show the concrete points maximizing the logit difference $y_2-y_1$ as a black $\times$ and the corresponding relaxation as a red \textcolor{red}{$\times$}.}%
\label{fig:ibp_example}
\end{wrapfigure}

%% file: method.tex
\section{Method -- Small Regions for Certified Training} \label{sec:method}
\vspace{-1.5mm}
To train networks that are not only robust and amenable to certification but also retain comparatively high standard accuracies, we propose the novel certified training method, \tool~--- \tooll. 
We leverage the key insight that computing an over-approximation of the worst-case loss over a small but carefully selected subset of the input region $\bc{B}^{\epsilon}(\vx)$ often yields a good proxy for the worst-case loss over the whole region while significantly reducing approximation errors. %

\begin{figure}[t]
	\centering
	\begin{adjustbox}{width=\textwidth}
		\input{figures/method_overview}
	\end{adjustbox}
	\vspace{-5mm}
	\caption[]{Illustration of \tool training. Instead of propagating a \boxd approximation (dashed box \tdbox{0.18}) of the whole input region (red \markerb{Wrong!65} and green \markerb{Correct!65} shapes in input space), \tool propagates a small subset of this region (solid box \tbox{0.18}), selected to contain the adversarial example (black $\times$) and thus the misclassified region (\markerb{Wrong!65} red). The smaller \boxd accumulates much fewer approximation errors during propagation, leading to a significantly smaller output relaxation, which induces much less regularization (medium blue \barrowm) than training with the full region (large blue \barrowl), but more than training with just the adversarial example (small blue \barrows).}
	\label{fig:method_overview}
	\vspace{-5mm}
\end{figure}
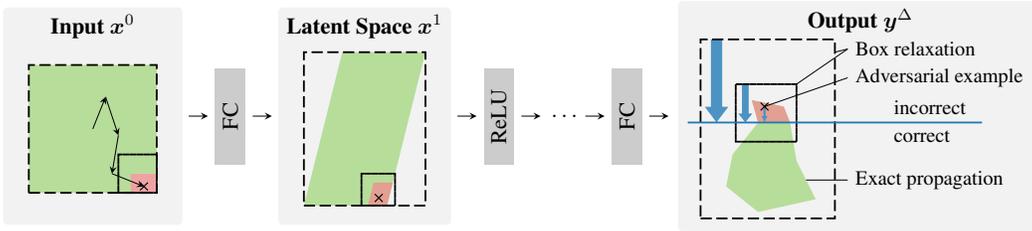

We illustrate this intuition in \cref{fig:method_overview}. Existing certified training methods always consider the whole input region (dashed box \tdbox{0.18} in the input panel). Propagating such large regions through the network yields quickly growing approximation errors and thus very imprecise over-approximations of the actual worst-case loss (compare the reachable set in red \markerb{Wrong!65} and green \markerb{Correct!65} to the dashed box \tdbox{0.18} in the output panel), causing significant over-regularization (large blue arrow \barrowl).
Adversarial training methods, in contrast, only consider individual points in the input space ($\times$ in \cref{fig:method_overview}) and often fail to capture the actual worst-case loss. This leads to insufficient regularization (small blue arrow \barrows in the output panel) and yields networks which are not amenable to certification and potentially not robust. 

We tackle this problem by propagating small, adversarially chosen subsets of the input region (solid box \tbox{0.18} in the input panel of \cref{fig:method_overview}), which we call the \emph{propagation region}. This leads to significantly reduced approximation errors (see the solid box \tbox{0.18} in the output panel) inducing a level of regularization in-between certified and adversarial training methods (medium blue arrow \barrowm), allowing us to train networks that are both robust and accurate.

More formally, we define an auxiliary objective for the robust optimization problem \cref{eq:rob_opt} as
\begin{equation}
\label{eq:sabr_opt}
\bc{L}_\text{\tool} = \max_{\vx^* \in \bc{B}^{\tau}(\vx')} \bc{L}_\text{CE}(\vx^*, t),
\end{equation}
where we replace the maximum over the whole input region $\bc{B}^{\epsilon}(\vx)$ with that over a carefully selected subset $\bc{B}^{\tau}(\vx')$. While choosing $\vx' = \Pi_{\bc{B}^{\epsilon -\tau}(\vx)} \argmax_{\vx^* \in \bc{B}^{\epsilon}(\vx)} \bc{L}_\text{CE}(\vx^*, t)$ would recover the original robust training problem (\cref{eq:rob_opt}), both, computing the maximum loss over a given input region (\cref{eq:sabr_opt}) and finding a point that realizes this loss is generally intractable. 
Instead, we instantiate \tool by combining different approximate approaches for the two key components: a) a method for choosing the location $\vx'$ and size $\tau$ of the propagation region, and b) a method used for propagating the thus selected region. Note that we thus generally do not obtain a sound over-approximation of the loss on $\bc{B}^\eps(\vx)$.
Depending on the size of the propagated region $\bc{B}^{\tau}(\vx')$, \tool can be seen as a continuous interpolation between adversarial training for infinitesimally small regions $\tau = 0$ and standard certified training for the full input region $\tau = \epsilon$.

\input{./figures/region_selection.tex}

\paragraph{Selecting the Propagation Region}
\tool aims to find and propagate a small subset of the adversarial input region $\bc{B}^{\epsilon}(\vx)$ that contains the inputs leading to the worst-case loss. 
To this end, we parametrize this propagation region as an $\ell_p$-norm ball $\bc{B}^{\tau}(\vx')$ with centre $\vx'$ and radius $\tau \leq \epsilon - \|\vx - \vx'\|_p$. %
We first choose $\tau = \ratios \epsilon$ by scaling the original perturbation radius $\epsilon$ with the \ratiol $\ratios \in (0,1]$. We then select $\vx'$ as follows: We conduct a PGD attack, choosing the preliminary centre $\vx^*$ as the sample with the highest loss. We then ensure the obtained region is fully contained in the original one by projecting $\vx^*$ onto $\bc{B}^{\epsilon - \tau}(\vx)$ to obtain $\vx'$. We show this in \cref{fig:region_selection}.

\paragraph{Propagation Method}
Having found the propagation region $\bc{B}^{\tau}(\vx')$, we can use any symbolic propagation method to compute an over-approximation of its worst-case loss. We chose \boxd propagation (\diffai \cite{MirmanGV18} or \ibp \citep{GowalIBP2018}) to obtain well-behaved optimization problems \citep{JovanovicBBV21}. There, choosing small propagation regions ($\tau \ll 1$), can significantly reduce the incurred over-approximation errors, as we will show later (see \cref{sec:theory}).

%% file: figures/method_overview.tex
\tikzset{
	point/.style={
		line width=0.5pt,,
		draw=black,
		cross out,
		inner sep=0pt,
		minimum width=3pt,
		minimum height=3pt,
	},
}

\begin{tikzpicture}
	\tikzstyle{toolstyle}=[dash pattern=on 3pt off 0pt]
	\tikzstyle{ibpstyle}=[dash pattern=on 5pt off 2pt]
    
    \def \dy{-0.2 cm}

    \begin{scope}[xshift=-0.4cm]   
        
        \node
		[fill=black!05, rectangle, rounded corners=2pt,
		minimum width=2.8cm, minimum height=3.4cm
		] (input_box) at (0.5, 1.0) {
			
		};
		\node[align=center, anchor=north] at (input_box.north) {\textbf{Input} $\vx^0$};
	    \begin{scope}[yshift=\dy]   
		        \draw[fill=Correct!80, opacity=0.8, draw=none] (-0.5, 0) -- (1.1, 0) -- (1.1, 0.3) -- (1.5, 0.3) -- (1.5, 2) -- (-0.5, 2) -- (-0.5, 0);
		    \draw[fill=Wrong!80, opacity=0.8, draw=none] (1.1, 0) -- (1.5, 0) -- (1.5, 0.3) -- (1.1, 0.3) -- (1.1, 0);

		    \draw[sbbm, thick, toolstyle] (0.9, 0) -- (1.5, 0) -- (1.5, 0.6) -- (0.9, 0.6) -- cycle;

		    \draw[ibpm,thick,ibpstyle] (-0.5, 0) -- (1.5, 0) -- (1.5, 2) -- (-0.5, 2) -- cycle;

    		\node[point] at (1.3,0.1) {};

		    \draw[-stealth] (0.5, 1) -- (0.7, 1.5);
		    \draw[-stealth] (0.7, 1.5) -- (0.9, 0.9);
		    \draw[-stealth] (0.9, 0.9) -- (0.8, 0.3);
		    \draw[-stealth] (0.8, 0.3) -- (1.3,0.1);
	    \end{scope}

    \end{scope}

    \draw[-stealth] (1.6, 1) -- (1.9, 1);
    \node[rectangle, fill, minimum width=1.5cm, align=center, very thick, color=netinside, text=black, rotate=90]
    at (2.25, 1) {FC};
    \draw[-stealth] (2.6, 1) -- (2.9, 1);
    
    \def\ang{atan(2.4/0.6)}
    
    \begin{scope}[xshift=0.4cm]
    
            \node
	[fill=black!05, rectangle, rounded corners=2pt,
	minimum width=2.7cm, minimum height=3.4cm
	] (intermediate_box) at (3.95, 1.0) {
	};
	\node[align=center, anchor=north] at (intermediate_box.north) {\textbf{Latent Space} $\vx^1$};

	\begin{scope}[yshift=\dy]   
        \draw[fill=Correct!80, opacity=0.8, draw=none] (3.0, -0.2) -- (4.3, -0.2) -- (4.9, 2.2) -- (3.6, 2.2) -- (3.0, -0.2);
    \draw[fill=Wrong!80, opacity=0.8, draw=none] (4.3, -0.2)--({4.3 + 0.36/tan(\ang)},0.16) -- ({4.0 + 0.36/tan(\ang)},0.16) -- (4.0, -0.2)--(4.3, -0.2);
    
    \draw[sbbm,thick, toolstyle] (3.9, -0.2) -- (3.9, 0.3) -- ({4.3+0.5/tan(\ang)}, 0.3) -- ({4.3+0.5/tan(\ang)}, -0.2) --(3.9, -0.2);
	\node[point] at (4.17,-0.08) {};
    
    \draw[ibpm,thick,ibpstyle] (3.0, -0.2) -- (4.9, -0.2) -- (4.9, 2.2) -- (3.0, 2.2) -- (3.0, -0.2);
    \end{scope}
    \end{scope}

    \begin{scope}[xshift=0.8cm]

    \draw[-stealth] (5.0, 1) -- (5.3, 1);
    \node[rectangle, fill, minimum width=1.5cm, align=center, very thick, color=netinside, text=black, rotate=90]
    at (5.65, 1) {ReLU};
    \draw[-stealth] (6.0, 1) -- (6.3, 1);

    \node[rectangle, align=center, very thick, text=black]
    at (6.7, 1) {$\dotsc$};
    
    \draw[-stealth] (7.0, 1) -- (7.3, 1);
    
    \node[rectangle, fill, minimum width=1.5cm, align=center, very thick, color=netinside, text=black, rotate=90]
    at (7.65, 1) {FC};
    \draw[-stealth] (8.0, 1) -- (8.3, 1);
    
    \end{scope}

    \begin{scope}[xshift=1cm]
    \node
	[fill=black!05, rectangle,     	rounded corners=2pt,
	minimum width=5.7cm, minimum height=3.6cm] (output) at (11.1, 1.0) {
	};
	\node[align=center, anchor=north] at (output.north) {\textbf{Output} $\vy^\Delta$};
    
    \def \yc{1.1}
    
 	\begin{scope}[yshift=\dy]   	
        \draw[fill=Correct!80, opacity=0.8, draw=none] (9.0, 0.1) -- (9.1, 0.6) -- (9.5, \yc)-- (9.4, 1.45) --(9.9, 1.35) -- (10.0, \yc) -- (10.1, 0.5) -- (10.4, -0.1) -- (9.5,-0.3)--(9.0, 0.1);
    \draw[fill=Wrong!80, opacity=0.8, draw=none] (9.5, \yc) -- (10.0, \yc) -- (9.9, 1.35) -- (9.4, 1.45) -- (9.5, \yc);
	\coordinate (adv)at (9.6,1.35);

    \draw[-{Triangle[width=10pt,length=7pt]}, line width=5pt,Sborder!80](8.85,2.4) -- (8.85, \yc);
    \draw[-{Triangle[width=6pt,length=4pt]}, line width=3pt, Sborder!80](9.3,1.7) -- (9.3, \yc);
    \draw[-{Latex[length=1mm,width=1mm]}, line width=1pt, Sborder!80](adv.center) -- (9.6, \yc);
 
    \node[point] () at (adv) {};
        
    \draw[ibpm,thick,ibpstyle] (8.6, -0.4) -- (10.7, -0.4) -- (10.7, 2.4) -- (8.6, 2.4) -- (8.6, -0.4);
    \draw[sbbm,thick, toolstyle] (10.1, 0.8) -- (10.1, 1.7) -- (9.15, 1.7) -- (9.15, 0.8) -- (10.1, 0.8);
    
    \draw[color=Sborder, thick] (8.4, \yc) -- (13, \yc);
    \node[font=\footnotesize, text=black, align=left, anchor=south west] at (11.5,\yc) {incorrect};
    \node[font=\footnotesize, text=black, align=left, anchor=north west] at (11.5,\yc) {correct};
    
    \draw[black] (10.1,1.7) -- (10.9, 2.25);
    \draw[black] (10.7,2.2) -- (10.9, 2.25)node[right, font=\footnotesize, text=black] {Box relaxation};
    \draw[black] (adv) -- (10.9, 1.8)node[right, font=\footnotesize, text=black] {Adversarial example};
    \draw[black] (10.25, 0.2) -- (10.9, 0.2) node[right, font=\footnotesize, text=black] {Exact propagation};
    \end{scope}
    \end{scope}

\end{tikzpicture}

%% file: figures/region_selection.tex
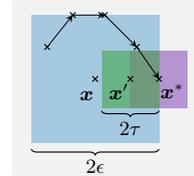
\begin{wrapfigure}[9]{r}{0.195 \textwidth}
\vspace{-5mm}
\centering
\tikzset{
	point/.style={
		line width = 0.5pt,
		draw=black,
		cross out,
		inner sep=0pt,
		minimum width=2pt,
		minimum height=2pt,
	},
	point_r/.style={
		line width = 0.5pt,
		draw=red,
		cross out,
		inner sep=0pt,
		minimum width=3pt,
		minimum height=3pt,
	},
}

\scalebox{0.85}{
\begin{tikzpicture}[scale=1.0]
\tikzset{>=latex}

        \node
		[fill=black!05, rectangle, rounded corners=2pt,
		minimum width=3.0cm, minimum height=2.8cm
		] (input_box) at (0.0, 0.0) {};

	    \begin{scope}[xshift=-2mm,yshift=1.5mm]
			\def \dl{1.}
			\def \ds{0.45}
			\def \da{-0.0}
			\def \dadv{0.5}
			
			\coordinate (adv) at (\dl, \da);
			
			\fill[fill=my-full-blue!40] (-\dl, -\dl) -- (-\dl, \dl) -- (\dl, \dl) -- (\dl, -\dl) -- cycle;
			\draw [line width=0.5pt, decoration={brace, mirror, raise=1mm}, decorate] (-\dl, -\dl) -- (\dl, -\dl);
			\node[anchor=north] at (0, -0.15 -\dl) {\footnotesize$2 \epsilon$};

			\fill[fill=my-full-purple!70, opacity=1] (\dl-\ds, \da-\ds) -- (\dl-\ds, \da+\ds) -- (\dl+\ds, \da+\ds) -- (\dl+\ds, \da-\ds) -- cycle;
			
			\fill[fill=my-full-green!100, opacity=0.5] (\dl-2*\ds, \da-\ds) -- (\dl-2*\ds, \da+\ds) -- (\dl, \da+\ds) -- (\dl, \da-\ds) -- cycle;
			\draw [line width=0.5pt, decoration={brace, mirror, raise=0.5mm}, decorate] (\dl-2*\ds, \da-\ds) -- (\dl, \da-\ds);
			\node[anchor=north] at (\dl-1*\ds, \da-\ds-0.1) {\footnotesize$2 \tau$};

			\node[point] (p0) at (\dl-3.5*\dadv,\da+1*\dadv) {};
			\node[point] (p1) at (\dl-2.7*\dadv,\da+2*\dadv) {};
			\node[point] (p2) at (\dl-1.7*\dadv,\da+2*\dadv) {};
			\node[point] (p3) at (\dl-0.7*\dadv,\da+\dadv) {};
			\node[point] at (adv) {};
			
			\node[point] (adv2) at (\dl-\ds,\da) {};
			
			\node[point] (O) at (0,0) {};
			\node[anchor=north east] at ($(O)+(0.1,-0.07)$) {\footnotesize$\vx$};
			\node[anchor=north west] at ($(adv)+(-0.1,0.05)$) {\footnotesize$\vx^*$};
			\node[anchor=north east] at ($(adv2)+(0.1,0.05)$) {\footnotesize$\vx'$};

			\draw[-stealth] (p0.center) -- (p1.center);
			\draw[-stealth] (p1.center) -- (p2.center);
			\draw[-stealth] (p2.center) -- (p3.center);
			\draw[-stealth] (p3.center) -- (adv.center);
		\end{scope}
\end{tikzpicture}
}
\vspace{-6mm}
\caption{Illustration of propagation region selection process.}%
\label{fig:region_selection}
\end{wrapfigure}

%% file: theory.tex
\section{Understanding \tool: Robust Loss and Growth of Small Boxes} \label{sec:theory}
In this section, we aim to uncover the reasons behind \tool's success. Towards this, we first analyse the relationship between robust loss and over-approximation size before investigating the growth of the \boxd approximation with propagation region size.

\paragraph{Robust Loss Analysis} Certified training typically optimizes an over-approximation of the worst-case cross-entropy loss $\bc{L}_\text{CE,rob}$, computed via the softmax of the upper-bound on the logit differences ${\vy}^\Delta := {\vy}-{y}_t$.  When training with the \boxd relaxation and assuming the target class $t=1$, w.l.o.g., we obtain the logit difference ${\vy}^\Delta \in [\bar{\vy}^\Delta - \bs{\delta}^\Delta, \bar{\vy}^\Delta + \bs{\delta}^\Delta]$ and thus the robust cross entropy loss
\begin{equation}\label{eq:rob_loss}
	\bc{L}_\text{CE, rob}(\vx) 
	= \ln\big(1 + \sum_{i=2}^{n} e^{\bar{y}^\Delta_i + \delta^\Delta_i}\big).
\end{equation}
We observe that samples with high ($>\!0$) worst-case misclassification margin $\bar{y}^\Delta + \delta^\Delta := \max_i \bar{y}^\Delta_i + \delta^\Delta_i$ dominate the overall loss and permit the per-sample loss term to be approximated as
\begin{equation}
\max_i \bar{y}^\Delta_i + \delta^\Delta_i =: \bar{y}^\Delta + \delta^\Delta < \bc{L}_\text{CE, rob} < \ln(n) + \max_i \bar{y}^\Delta_i + \delta^\Delta_i.
\end{equation}
Further, we note that the \boxd relaxations of many functions preserve the box centres, i.e., %
$\bar{\vx}^i = \vf(\bar{\vx}^{i-1})$. Only unstable ReLUs, i.e., ReLUs containing $0$ in their input bounds, introduce a slight shift. However, these are empirically few in certifiably trained networks (see \cref{tab:rel_states}).

These observations allow us to decompose the robust loss into an accuracy term $\bar{y}^\Delta$, corresponding to the misclassification margin of the adversarial example $\vx'$ at the centre of the propagation region, and a robustness term ${\delta}^\Delta$, bounding the difference to the actual worst-case loss.
These terms generally represent conflicting objectives, as local robustness requires the network to disregard high frequency features \citep{IlyasSTETM19}.
Therefore, robustness and accuracy are balanced to minimize the optimization objective \cref{eq:rob_loss}. Consequently, reducing the regularization induced by the robustness term will bias the optimization process towards standard accuracy. Next, we investigate how \tool reduces exactly this regularization strength, by propagating smaller regions.

\begin{wrapfigure}[31]{r}{0.34 \textwidth}
	\vspace{-4mm}
	\centering
	\includegraphics[width=1.0\linewidth]{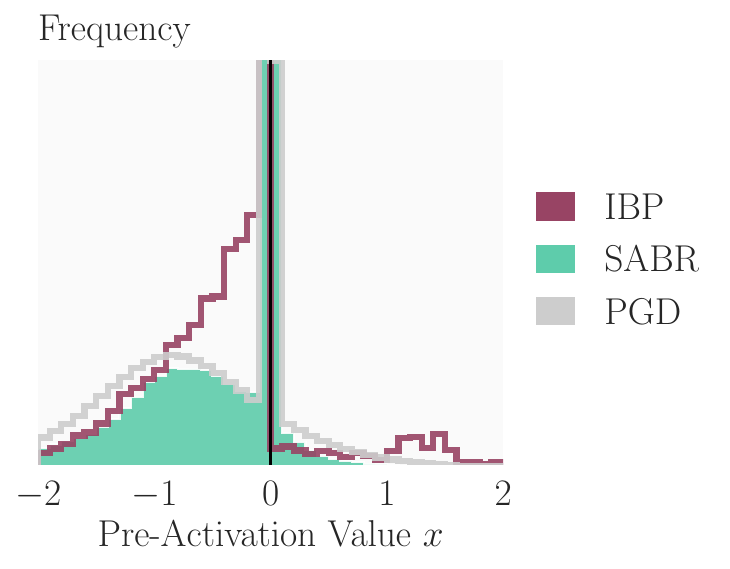}
	\vspace{-6mm}
	\caption{Input distribution for last ReLU layer depending on training method.}%
	\label{fig:pre_act_density}
	\vspace{5mm}
	\includegraphics[width=0.95\linewidth]{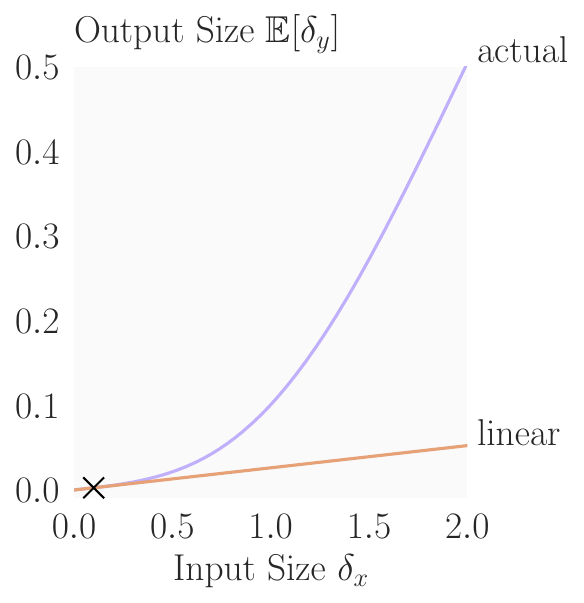}
	\vspace{-2mm}
	\caption{Comparison of the actual (purple) mean output size and a linear growth (orange) around the black $\times$ for a ReLU layer where input box centres $\bar{x}\! \sim\! \bc{N}(\mu\!=\!-1.0,\sigma\!=\!0.5)$.}%
	\label{fig:theory_slopes}
\end{wrapfigure}
\paragraph{\boxd Growth}
To investigate how \boxd approximations grow as they are propagated, let us again consider an $L$-layer network $\vh = \vf_L \circ \bs{\sigma} \circ \vf_{L-2} \circ \dotsc \circ \vf_1$, with linear layers $\vf_i$ and ReLU activation functions $\bs{\sigma}$. Given a \boxd input with radius $\delta^{i-1}$ and centre distribution $\bar{x}^{i-1} \sim \bc{D}$, we now define the per-layer growth rate $\kappa^i$ as the ratio of input and expected output radius:
\begin{equation}
	\kappa^i = \frac{\E_\bc{D}[ \delta^i ]}{\delta^{i-1}}.
\end{equation}
For linear layers with weight matrix $\mW$, we obtain an output radius $\delta^{i} = |\mW| \delta^{i-1}$ and thus a constant growth rate $\kappa^i$, corresponding to the row-wise $\ell_1$ norm of the weight matrix $|\mW_{j,\cdot}|_1$. Empirically, we find most linear and convolutional layers to exhibit growth rates between $10$ and $100$ (see \cref{tab:growth_rates} in \cref{app:growth_rates}).

For ReLU layers $\vx^i = \sigma(\vx^{i-1})$, computing the growth rate is more challenging, as it depends on the location and size of the inputs. \citet{ShiBetterInit2021} assume the input \boxd centres $\bar{\vx}^{i-1}$ to be symmetrically distributed around $0$, i.e., $P_\bc{D}(\bar{x}^{i-1}) = P_\bc{D}(-\bar{x}^{i-1})$, and obtain a constant growth rate of $\kappa^i = 0.5$. While this assumption holds at initialization, we observe that trained networks tend to have more inactive than active ReLUs (see \cref{tab:rel_states}), indicating asymmetric distributions with more negative inputs (see \cref{fig:pre_act_density}).

We now investigate this more realistic setting. We first consider the two limit cases where input radii $\delta^{i-1}$ go against $0$ and $\infty$. 
When input radii are $\delta^{i-1} \approx 0$, active neurons will stay stably active, yielding $\delta^{i} = \delta^{i-1}$ and inactive neurons will stay stably inactive, yielding $\delta^{i}=0$. Thus, we obtain a growth rate, equivalent to the portion of active neurons.
In the other extreme $\delta^{i-1} \rightarrow \infty$, all neurons will become unstable with $\bar{x}^{i-1} \ll \delta^{i-1}$, yielding $\delta^{i} \approx 0.5 \, \delta^{i-1}$, and thus a constant growth rate of $\kappa^i = 0.5$.
To analyze the behavior in between those extremes, we assume pointwise asymmetry favouring negative inputs, i.e., $p(\bar{x}^{i-1} = -z) > p(\bar{x}^{i-1} = z), \; \forall z \in \R^{>0}$. In this setting, we find that output radii grow strictly super-linear in the input size:

\begin{restatable}[Hyper-Box Growth]{thm}{growth}
	\label{thm:growth}
	Let $y := \sigma(x) = \max(0,x)$ be a ReLU function and consider box inputs with radius $\delta_x$ and asymmetrically distributed centres $\bar{x} \sim \bc{D}$ such that $P_\bc{D}(\bar{x} = -z) > P_\bc{D}(\bar{x} = z), \; \forall z\in \R^{>0}$. Then, the mean output radius $\delta_y$ will grow super-linearly in the input radius $\delta_x$. More formally:
	\begin{equation}
	\forall \delta_{x},\delta_{x}' \in \R^{\geq0} \colon \quad \delta_{x}' > \delta_{x} \implies \E_\bc{D} [\delta'_{y}] > \E_\bc{D} [\delta_{y}] + (\delta'_{x} - \delta_{x}) \frac{\partial}{\partial \delta_{x}} \E_\bc{D} [\delta_{y}].
	\end{equation}
\end{restatable}

We defer a proof to \cref{app:theory_proof} and illustrate this behaviour in \cref{fig:theory_slopes} for the box centre distribution $\bar{x} \sim \bc{N}(\mu\!=\!-1.0,\sigma\!=\!0.5)$. There, we clearly observe that the actual super-linear growth (purple) outpaces a linear approximation (orange). While even the qualitative behaviour depends on the exact centre distribution and the input box size $\delta_x$, we can solve special cases analytically. For example, a piecewise uniform centre distribution yields quadratic growth on its support (see \cref{app:theory_proof}).

Multiplying all layer-wise growth rates, we obtain the overall growth rate $\kappa \!= \!\prod_{i=1}^{L} \kappa^i$, which is exponential in network depth and super-linear in input radius. When not specifically training with the \boxd relaxation, we empirically observe that the large growth factors of linear layers dominate the shrinking effect of the ReLU layers, leading to a quick exponential growth in network depth.
Further, for both \tool and \ibp trained networks, the super-linear growth in input radius  empirically manifests as exponential behaviour (see \cref{fig:error_plot,fig:error_plot_app}). Using \tool, we thus expect the regularization induced by the robustness term to decrease super-linearly, and empirically even exponentially, with \ratiol $\ratios$, explaining the significantly higher accuracies compared to \ibp.

%% file: experiments.tex
\section{Evaluation} \label{sec:experiments}
\vspace{-1mm}
In this section, we first compare \tool to existing certified training methods before investigating its behavior in an ablation study.

\vspace{-1mm}
\paragraph{Experimental Setup} We implement \tool in PyTorch \citep{PaszkeGMLBCKLGA19}\footnote{Code released at \url{https://github.com/eth-sri/sabr}} and use \mnbab \citep{FerrariMJV22} for certification. We conduct experiments on \mnist \citep{lecun2010mnist}, \cifar \citep{krizhevsky2009learning}, and \TIN \citep{Le2015TinyIV} for the challenging $\ell_\infty$ perturbations, using the same 7-layer convolutional architecture \cnns as prior work \citep{ShiBetterInit2021} unless indicated otherwise (see \cref{app:eval_detail} for more details). We choose similar training hyper-parameters as prior work \citep{ShiBetterInit2021} and provide more detailed information in \cref{app:eval_detail}.

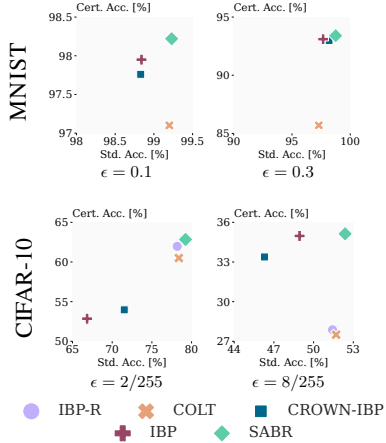
\begin{wrapfigure}[18]{r}{0.38\textwidth}
	\centering
	\vspace{-7mm}
	\input{figures/scatter_results}
	\vspace{-7mm}
	\caption{Certified over standard accuracy for different certified training methods. The upper right-hand corner is best.}
	\label{fig:results_scatter}
\end{wrapfigure}

\subsection{Main Results}
\vspace{-1mm}
We compare \tool to state-of-the-art certified training methods in \cref{tab:results} and \cref{fig:results_scatter}, reporting the best results achieved with a given method on \emph{any} architecture.%

In \cref{fig:results_scatter}, we show certified over standard accuracy (upper right-hand corner is best) and observe that \tool (\markersbb) dominates all other methods, achieving both the highest certified and standard accuracy across all settings. As existing methods typically perform well either at large \emph{or} small perturbation radii (see \cref{tab:results,fig:results_scatter}), we believe the high performance of \tool \emph{across perturbation radii} to be particularly promising.

Methods striving to balance accuracy and regularization by bridging the gap between provable and adversarial training (\markercolt, \markeribpr)\citep{BalunovicV20,PalmaIBPR22} perform only slightly worse than \tool at small perturbation radii (\cifar $\epsilon=2/255$), but much worse at large radii, e.g., attaining only $27.5\%$ (\markercolt) and $27.9\%$ (\markeribpr) certifiable accuracy for \cifar $\epsilon \!=\! 8/255$ compared to $35.1\%$ (\markersbb). %
Similarly, methods focusing purely on certified accuracy by directly optimizing over-approximations of the worst-case loss (\markeribp, \markerccibp) \citep{GowalIBP2018,ZhangCXGSLBH20} tend to perform well at large perturbation radii (\mnist $\epsilon\!=\!0.3$ and \cifar $\epsilon\!=\!8/255$), but poorly at small perturbation radii, e.g. on \cifar at $\epsilon\!=\!2/255$, \tool improves natural accuracy to $79.2\%$ (\markersbb) up from $66.8\%$ (\markeribp) and $71.5\%$ (\markerccibp) and even more significantly certified accuracy to $62.8\%$ (\markersbb) up from $52.9\%$ (\markeribp) and $54.0\%$ (\markerccibp). On the particularly challenging \TIN, \tool again dominates all existing certified training methods, improving certified and standard accuracy by almost $3\%$.

To summarize, \tool improves strictly on all existing certified training methods across all commonly used benchmarks with relative improvements exceeding $25\%$ in some cases. 
\begin{table}[t]
	\centering	
	\begin{adjustbox}{width=\columnwidth,center}
		\begin{threeparttable}
			\caption{Comparison of the standard (Acc.) and certified (Cert. Acc.) accuracy for different certified training methods on the full \mnist, \cifar, and \TIN test sets. We use \mnbab \citep{FerrariMJV22} for certification and report other results from the relevant literature.			\vspace{-2mm}}
			\begin{tabular}{lclccc}
				\toprule
				Dataset & $\epsilon_\infty$ & Training Method & Source & Acc. [\%]  & Cert. Acc. [\%] \\
				\midrule
				\multirow{8}*{\mnist}&\multirow{4}*{0.1}&\colt &\citet{BalunovicV20}     & 99.2  & 97.1 \\
				&&\crownibp &\citet{ZhangCXGSLBH20}& 98.83 &  97.76 \\
				&&\ibp &\citet{ShiBetterInit2021}  & 98.84 & 97.95 \\
				&&\tool &this work            &\textbf{99.23} & \textbf{98.22}  \\
				\cmidrule(rl){2-6}
				&\multirow{4}*{0.3}                  &\colt &\citet{BalunovicV20}        & 97.3  & 85.7\\
				&&\crownibp &\citet{ZhangCXGSLBH20}  & 98.18 & 92.98 \\
				&&\ibp &\citet{ShiBetterInit2021}    & 97.67 & 93.10 \\ 
				&&\tool &this work              & \textbf{98.75} & \textbf{93.40} \\
				\cmidrule(rl){1-6}
				\multirow{12}*{\cifar} &\multirow{5}*{2/255} &\colt &\citet{BalunovicV20}                  & 78.4 & 60.5 \\
				&&\crownibp &\citet{ZhangCXGSLBH20}            & 71.52 & 53.97 \\
				&&\ibp &\citet{ShiBetterInit2021}              & 66.84 & 52.85 \\
				&&\ibpr&\citet{PalmaIBPR22}  & 78.19 & 61.97\\
				&&\tool &this work         & \textbf{79.24} & \textbf{62.84}\\
				\cmidrule(rl){2-6}
				&\multirow{6}*{8/255}&\colt  &\citet{BalunovicV20}                        & 51.7 & 27.5 \\
				&&\crownibp &\citet{XuS0WCHKLH20}                    & 46.29 & 33.38\\
				&&\ibp &\citet{ShiBetterInit2021}                     & 48.94 & 34.97 \\ 
				&&\ibpr &\citet{PalmaIBPR22}      &51.43 & 27.87\\
				&&\tool &this work                             & \textbf{52.38} & \textbf{35.13}\\ 
				\cmidrule(rl){1-6}
				\multirow{3}*{\TIN}&\multirow{3}*{1/255}& \crownibp &\citet{ShiBetterInit2021}  & 25.62 & 17.93\\
				&&\ibp &\citet{ShiBetterInit2021} & 25.92 &17.87 \\
				&& \tool &this work & \textbf{28.85} & \textbf{20.46} \\
				\bottomrule
			\end{tabular}
			\label{tab:results}
		\end{threeparttable}
	\end{adjustbox}
	\vspace{-3mm}
\end{table}

\begin{wraptable}[11]{r}{0.52\textwidth}
	\centering
	\vspace{-0mm}
	\begin{minipage}{1.00\linewidth}
	\resizebox{1.00\linewidth}{!}{
	\renewcommand{\arraystretch}{1.0}
		\begin{threeparttable}
			\caption{Comparison of natural (Nat.) and certified (Cert.) accuracy [\%] to \sortnet \citep{ZhangJHW22}.\vspace{-2mm}}
			\label{tab:linf_results}
			\begin{tabular}{@{}lccccc@{}}
				\toprule
				\multirow{2.5}{*}{Dataset} & \multirow{2.5}{*}{$\epsilon$} & \multicolumn{2}{c}{\sortnet} & \multicolumn{2}{c}{\tool (\textbf{ours})} \\
				\cmidrule(rl){3-4}				\cmidrule(rl){5-6}
				&& Nat. & Cert. & Nat. & Cert. \\
				\midrule
				\multirow{2}*{\mnist}& 0.1 & 99.01 & 98.14 & \textbf{99.23} &  \textbf{98.22}\\
				& 0.3 & 98.46 & \textbf{93.40} & \textbf{98.75} &  \textbf{93.40}\\
				\cmidrule(rl){1-6}
				\multirow{2}*{\cifar}& 2/255 & 67.72 & 56.94 & \textbf{79.24} &  \textbf{62.84}\\
				& 8/255 & \textbf{54.84} & \textbf{40.39} & 52.38 &  35.13\\
				\cmidrule(rl){1-6}
				\TIN & 1/255 & 25.69 & 18.18 & \textbf{28.85} &  \textbf{20.46}\\
				\bottomrule
			\end{tabular}
		\end{threeparttable}
	}
	\end{minipage}
\end{wraptable}
In contrast to certified training methods, \citet{ZhangJHW22} propose \sortnet, a generalization of recent architectures \citep{ZhangCLHW21,ZhangJH022,AnilLG19} with inherent $\ell_\infty$-robustness properties. While \sortnet performs well at very high perturbation magnitudes ($\epsilon=8/255$ for \cifar), it is dominated by \tool in all other settings. Further, robustness can only be obtained against one perturbation type at a time.

\subsection{Ablation Studies}
\begin{figure}[t]
	\vspace{2mm}
	\begin{minipage}{1\textwidth}
	\begin{subfigure}[b]{0.29\textwidth}
		\centering
		\includegraphics[width=\textwidth]{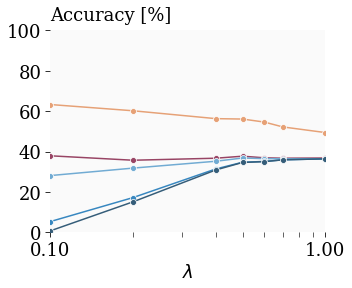}
		\vspace{-5mm}
		\caption{\cnns\\$\epsilon=8/255$}
		\label{fig:ablation_lambda_8_255}
	\end{subfigure}
	\hfill
	\begin{subfigure}[b]{0.29\textwidth}
		\centering
		\includegraphics[width=\textwidth]{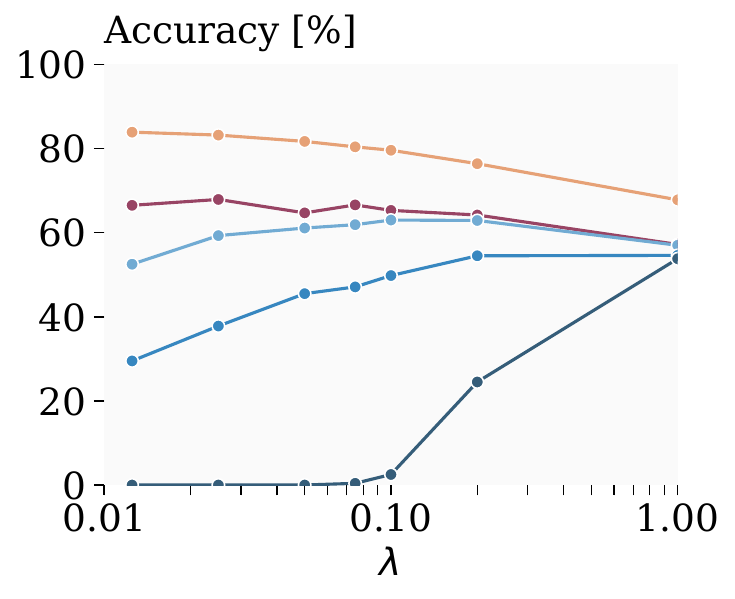}
		\vspace{-5mm}
		\caption{\cnns\\$\epsilon=2/255$}
		\label{fig:ablation_lambda_2_255}
	\end{subfigure}
	\hfill
	\begin{subfigure}[b]{0.408\textwidth}
		\centering
		\includegraphics[width=\textwidth]{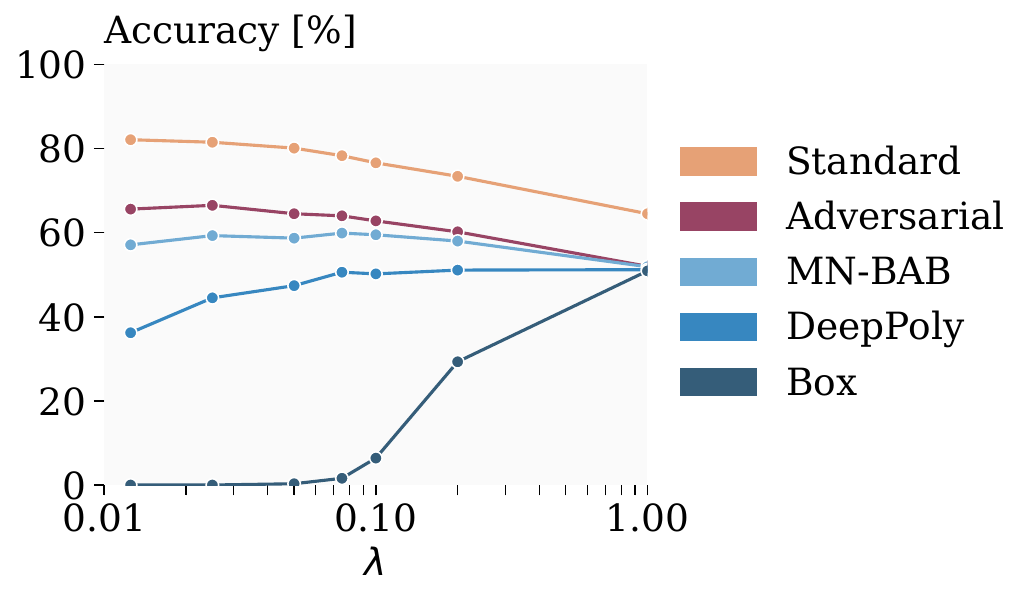}
		\vspace{-5mm}
		\caption{\cnnsn$\qquad\qquad\;\;\;$\\$\;\epsilon=2/255\qquad\qquad\;\;\;$}
		\label{fig:ablation_lambda_narrow}
	\end{subfigure}
	\end{minipage}
	\vspace{-3mm}
	\caption{Standard, adversarial and certified accuracy depending on the certification method (\boxd, \deeppoly, and \mnbab) for the first 1000 test set samples of \cifar.}
	\label{fig:ablation_lambda}
	\vspace{-4mm}
\end{figure}
\paragraph{Certification Method and Propagation Region Size}
To analyze the interaction between the precision of the certification method and the size of the propagation region, we train a range of models with \ratiols $\ratios$ varying from $0.0125$ to $1.0$ and analyze them with verification methods of increasing precision (\boxd, \deeppoly, \mnbab). Further, we compute adversarial accuracies using a 50-step PGD attack \citep{MadryMSTV18} with 5 random restarts and the targeted logit margin loss \citep{Carlini017}. 
We illustrate results in \cref{fig:ablation_lambda} and observe that standard and adversarial accuracies increase with decreasing $\ratios$, as regularization decreases. %
For $\ratios = 1$, i.e., \ibp training, we observe little difference between the verification methods. However, as we decrease $\ratios$, the \boxd verified accuracy decreases quickly, despite \boxd  relaxations being used during training.
In contrast, using the most precise method, \mnbab, we initially observe increasing certified accuracies, as the reduced regularization yields more accurate networks, before the level of regularization becomes insufficient for certification. While \deeppoly loses precision less quickly than \boxd, it can not benefit from more accurate networks. This indicates that the increased accuracy, enabled by the reduced regularization, may rely on complex neuron interactions, only captured by \mnbab.
These trends hold across perturbation magnitudes (\cref{fig:ablation_lambda_2_255,fig:ablation_lambda_8_255}) and become even more pronounced for narrower networks (\cref{fig:ablation_lambda_narrow}), which are more easily over-regularized.

This qualitatively different behavior depending on the precision of the certification method highlights the importance of recent advances in neural network verification for certified training. Even more importantly, these results clearly show that provably robust networks do not necessarily require the level of regularization introduced by \ibp training.

\begin{wrapfigure}[15]{r}{0.44\textwidth}
	\centering
	\vspace{-5mm}
	\includegraphics[width=0.95\linewidth]{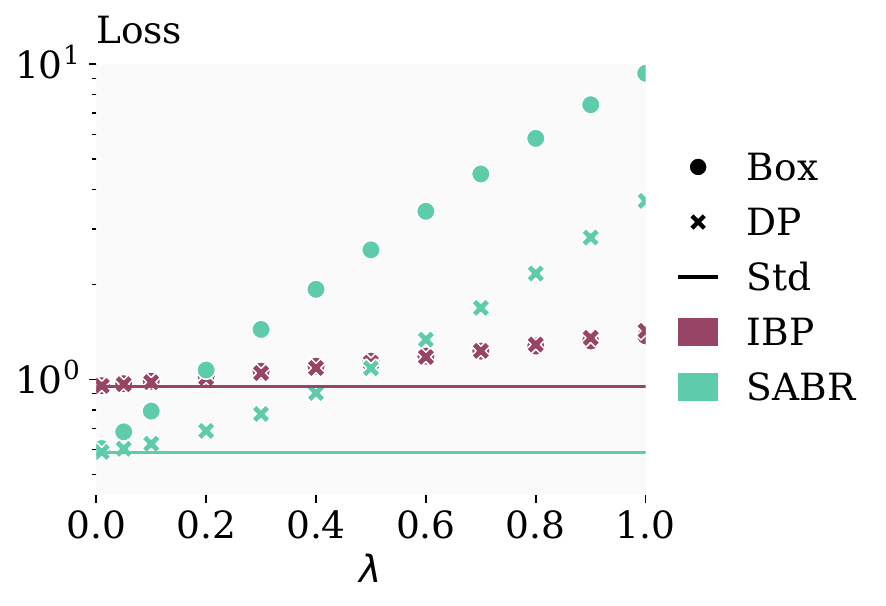}
	\vspace{-3mm}
	\caption{Standard (Std.) and robust cross-entropy loss, computed with \boxd (Box) and \deeppoly (DP) for an \ibp and  \tool trained network over evaluation  \ratiols $\ratios$.}
	\label{fig:error_plot}
\end{wrapfigure}
\paragraph{Loss Analysis}
In \cref{fig:error_plot}, we compare the robust loss of a \tool and an \ibp trained network across different propagation region sizes (all centred around the original sample) depending on the bound propagation method used. We first observe that, when propagating the full input region ($\ratios=1$), the \tool trained network yields a much higher robust loss than the \ibp trained one. However, when comparing the respective training \ratiols, $\ratios=0.05$ for \tool and $\ratios=1.0$ for \ibp, \tool yields significantly smaller training losses. Even more importantly, the difference between robust and standard loss is significantly lower, which, recalling \cref{sec:theory}, directly corresponds to a reduced regularization for robustness and allows the \tool trained network to reach a much lower standard loss. 
Finally, we observe the losses to clearly grow super-linearly with increasing propagation region sizes (note the logarithmic scaling of the y-axis) when using the \boxd relaxation, agreeing well with our theoretical results in \cref{sec:theory}. While the more precise \deeppoly (DP) bounds yield significantly reduced robust losses for the \tool trained network, the \ibp trained network does not benefit at all, again highlighting its over-regularization. See \cref{app:eval_detail} for extended results.

\begin{wraptable}[12]{r}{0.26 \textwidth}    
	\centering
	\vspace{-0.5mm}
	\begin{minipage}{1.00\linewidth}
		\resizebox{1.00\linewidth}{!}{
		\begin{threeparttable}
			\caption{Cosine similarity between $\nabla_\theta \bc{L}_\text{rob}$ for \ibp and \tool and $\nabla_\theta \bc{L}_\text{CE}$ for adversarial (Adv.) and unperturbed (Std.) examples.
				\vspace{-2mm}}%
			\label{tab:cos_sim}
			\centering
			\begin{tabular}{lcc}
				\toprule
				Loss & \ibp & \tool \\ 
				\midrule
				Std.	& $0.5586$ & $\mathbf{0.8071}$ \\	
				Adv. & $0.8047$ & $\mathbf{0.9062}$ \\
				\bottomrule
			\end{tabular}
		\end{threeparttable}
	}
	\end{minipage}
\end{wraptable}
\paragraph{Gradient Alignment}
To analyze whether \tool training is actually more aligned with standard accuracy and empirical robustness, as indicated by our theory in \cref{sec:theory}, we conduct the following experiment for \cifar and $\epsilon=2/255$: We train one network using \tool with $\ratios=0.05$ and one with \ibp, corresponding to $\ratios=1.0$. For both, we now compute the gradients $\nabla_\theta$ of their respective robust training losses $\bc{L}_\text{rob}$ and the cross-entropy loss $\bc{L}_\text{CE}$ applied to unperturbed (Std.) and adversarial (Adv.) samples. We then report the mean cosine similarity between these gradients across the whole test set in \cref{tab:cos_sim}. We clearly observe that the \tool loss is much better aligned with both the cross-entropy loss of unperturbed and adversarial samples, corresponding to standard accuracy and empirical robustness, respectively.

\begin{wraptable}[9]{r}{0.42\textwidth}
	\centering
	\vspace{-4mm}
	\begin{minipage}{1.00\linewidth}
		\resizebox{1.00\linewidth}{!}{
	\begin{threeparttable}
		\caption{Average percentage of active, inactive, and unstable ReLUs for concrete points and boxes depending on training method.	\vspace{-3mm}}
		\label{tab:rel_states}
		\begin{tabular}{lrrrrr}
			\toprule
			& \multicolumn{2}{c}{Point} & \multicolumn{3}{c}{Whole Region} \\ 
			\cmidrule(rl){2-3}  \cmidrule(rl){4-6}
			Method& Act & Inact & Unst & Act & Inact  \\
			\midrule
			IBP & 26.2 & 73.8 & 1.18&25.6&73.2 \\
			\tool & 35.9 & 64.1 &  3.67 & 34.3 & 62.0 \\
			PGD & 36.5& 63.5 & 65.5&15.2&19.3 \\
			\bottomrule
		\end{tabular}
	\end{threeparttable}
	}
\end{minipage}

\end{wraptable}

\paragraph{ReLU Activation States} 
The portion of ReLU activations which are (stably) active, inactive, or unstable has been identified as an important characteristic of certifiably trained networks \citep{ShiBetterInit2021}. We evaluate these metrics for \ibp, \tool, and adversarially (PGD) trained networks on \cifar at $\epsilon=2/255$, using the \boxd relaxation to compute intermediate bounds, and report the average over all layers and test set samples in \cref{tab:rel_states}. We observe that, when evaluated on concrete points, the \tool trained network has around 37\% more active ReLUs than the \ibp trained one and almost as many as the PGD trained one, indicating a significantly smaller level of regularization. While the \tool trained network has around $3$-times as many unstable ReLUs as the \ibp trained network, when evaluated on the whole input region, it has $20$-times fewer than the PGD trained one, highlighting the improved certifiability.

%% file: figures/scatter_results.tex
\begin{tikzpicture}[] 

\def \dx{2.15}
\def \dy{-2.75}
\def \ds{0.38}
\def \dds{0.41}

\node (a) at (0,0) {\includegraphics[width=\dds\linewidth]{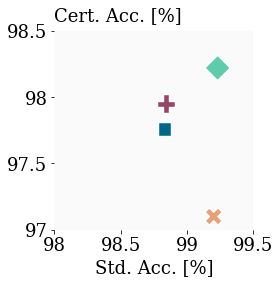}};
\node (b) at (\dx,0) {\includegraphics[width=\ds\linewidth]{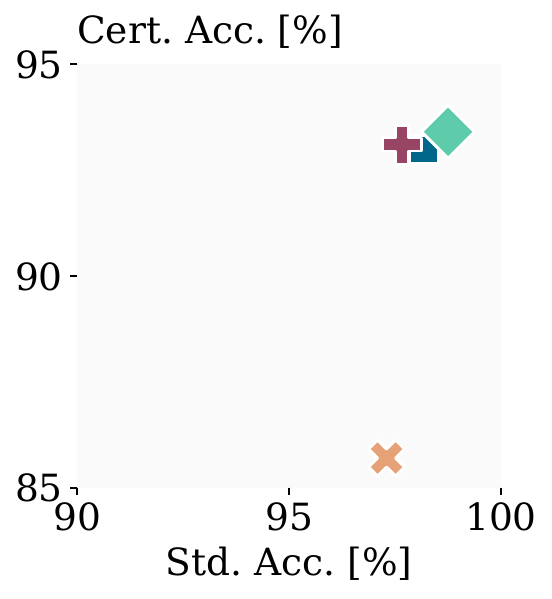}};
\node (c) at (0,\dy) {\includegraphics[width=\ds\linewidth]{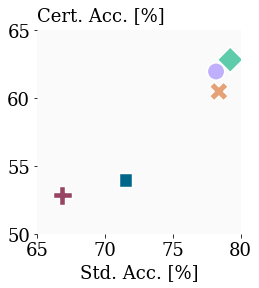}};
\node (d) at (\dx,\dy) {\includegraphics[width=\ds\linewidth]{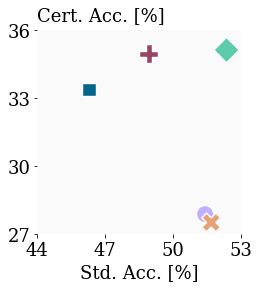}};

\node[anchor=south,rotate=90] at (a.west) {\footnotesize \mnist};
\node[anchor=south,rotate=90] at (c.west) {\footnotesize \cifar};

\node[anchor=north,rotate=0, scale=0.7] at ($(a.south)+(0,0.18)$) {\footnotesize $\epsilon=0.1$};
\node[anchor=north,rotate=0, scale=0.7] at ($(b.south)+(0,0.18)$) {\footnotesize $\epsilon=0.3$};
\node[anchor=north,rotate=0, scale=0.7] at ($(c.south)+(0,0.18)$) {\footnotesize $\epsilon=2/255$};
\node[anchor=north,rotate=0, scale=0.7] at ($(d.south)+(0,0.18)$) {\footnotesize $\epsilon=8/255$};

\node[anchor=north] (e) at ($(c.south)!0.50!(d.south)+(-0.1,-0.1)$) {\includegraphics[width=0.95\linewidth]{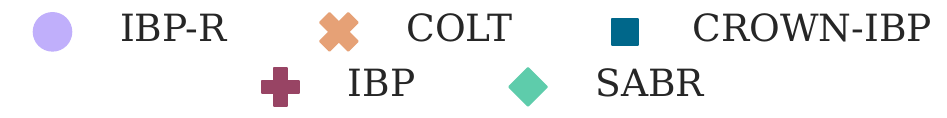}};
\end{tikzpicture}

%% file: related_work.tex
\section{Related Work} \label{sec:related_work}
\vspace*{-2mm}

\paragraph{Verification Methods} Deterministic verification methods analyse a given network by using abstract interpretation \citep{GehrMDTCV18,SinghGMPV18,SinghGPV19}, or translating the verification into an optimization problem which they then solve using linear programming (LP) \citep{PalmaBBTK21,MullerMSPV22,WangZXLJHK21,ZhangWXLLJ22}, mixed integer linear programming (MILP) \citep{TjengXT19,SinghGPV19A}, or semidefinite programming (SDP) \citep{RaghunathanSL18,DathathriDKRUBS20}. However, as neural network verification is generally NP-complete \citep{KatzBDJK17}, many of these methods trade precision for scalability, yielding so-called \emph{incomplete} certification methods, which might fail to prove robustness even when it holds.
In this work, we analyze our \tool trained networks with deterministic methods.

\paragraph{Certified Training} \diffai \citep{MirmanGV18} and \ibp \citep{GowalIBP2018} minimize a sound over-approximation of the worst-case loss computed using the \boxd relaxation. 
\citet{WongSMK18} instead use the \deepz relaxation \citep{SinghGMPV18}, approximated using Cauchy random matrices.
\citet{WongK18} compute worst-case losses by back-substituting linear bounds using fixed relaxations.  \crownibp \citep{ZhangCXGSLBH20} uses a similar back-substitution approach but leverages minimal area relaxations introduced by \citet{ZhangWCHD18} and \citet{SinghGPV19} to bound the worst-case loss while computing intermediate bounds using the less precise but much faster \boxd relaxation.
\citet{ShiBetterInit2021} show that they can obtain the same accuracies with much shorter training schedules by combining \ibp training with a special initialization.
\colt \citep{BalunovicV20} combines propagation using the \deepz relaxation with adversarial search.
\ibpr \citep{PalmaIBPR22} combines adversarial training with much larger perturbation radii and a ReLU-stability regularization based on the \boxd relaxation. 
We compare favorably to all (recent) methods above in our experimental evaluation (see \cref{sec:experiments}). \citet{MullerBV21} combine certifiable and accurate networks to allow for more efficient trade-offs between robustness and accuracy.

The idea of propagating subsets of the adversarial input region has been explored in the settings of adversarial patches \citep{ChiangNAZSG20} and geometric perturbations \citep{BalunovicBSGV19}, where the number of subsets required to cover the whole region is linear or constant in the input dimensionality. However, these methods are not applicable to the $\ell_p$-perturbation setting, we consider, where this scaling is exponential. 

\paragraph{Robustness by Construction}
\citet{LiCWC19}, \citet{LecuyerAG0J19}, and \citet{CohenRK19} construct locally Lipschitz classifiers by introducing randomness into the inference process, allowing them to derive probabilistic robustness guarantees. Extended in a variety of ways \citep{SalmanLRZZBY19,YangDHSR020}, these methods can obtain strong robustness guarantees with high probability \citep{SalmanLRZZBY19} at the cost of significantly (100x) increased runtime during inference. We focus our comparison on deterministic methods.
\citet{ZhangCLHW21} propose a novel architecture, which inherently exhibits $\ell_\infty$-Lipschitzness properties, allowing them to efficiently derive corresponding robustness guarantees. \citet{ZhangJH022} build on this work by improving the challenging training process. Finally, \citet{ZhangJHW22} generalize this concept in \sortnet.

%% file: conclusion.tex
\section{Conclusion} \label{sec:conclusion}
\vspace*{-2mm}
We introduced a novel certified training method called \tool (\tooll) based on the key insight, that propagating small but carefully selected subsets of the input region combines small approximation errors and thus regularization with well-behaved optimization problems. This allows \tool trained networks to outperform \emph{all} existing certified training methods on \emph{all} commonly used benchmarks in terms of \emph{both} standard and certified accuracy.
Even more importantly, \tool lays the foundation for a new class of certified training methods promising to alleviate the robustness-accuracy trade-off and enable the training of networks that are both accurate and certifiably robust.

%% file: statements.tex
\section{Ethics Statement}
As \tool improves both certified and standard accuracy compared to existing approaches, it could help make real-world AI systems more robust to both malicious and random interference. Thus any positive and negative societal effects these systems have could be amplified. Further, while we achieve state-of-the-art results on all considered benchmark problems, this does not (necessarily) indicate sufficient robustness for safety-critical real-world applications, but could give practitioners a false sense of security when using \tool trained models.

\section{Reproducibility Statement}
We publish our code, all trained models, and detailed instructions on how to reproduce our results at \url{https://github.com/eth-sri/sabr}, providing an anonymized version to the reviewers. 
Further, we provide proofs for our theoretical contributions in \cref{app:theory_proof} and a detailed description of all hyper-parameter choices as well as a discussion of the used data sets including all preprocessing steps in \cref{app:eval_detail}.

%% file: acknowledgements.tex
\section*{Acknowledgements}
This work is supported in part by ELSA --- European Lighthouse on Secure and Safe AI funded by the European Union under grant agreement No. 101070617. Views and opinions expressed are however those of the authors only and do not necessarily reflect those of the European Union or European Commission. Neither the European Union nor the European Commission can be held responsible for them.

This work has received funding from the Swiss State Secretariat for Education, Research and Innovation (SERI) (SERI-funded ERC Consolidator Grant).

%% file: appendix.tex
\input{theory_appendix}
\input{eval_appendix}

%% file: theory_appendix.tex
\section{Deferred Proofs}\label{app:theory_proof}
In this section, we provide the proof for \cref{thm:growth_lem}. Let us first consider the following Lemma:

\begin{restatable}[Hyper-Box Growth]{lem}{growth_lem}
	\label{thm:growth_lem}
	Let $y := \sigma(x) = \max(0,x)$ be a ReLU function and consider box inputs with radius $\delta_x$ and centres $\bar{x} \sim \bc{D}$. Then the mean radius $\E\delta_y$ of the output boxes will satisfy:
	\begin{equation}
		\frac{\partial}{\partial\delta_{x,i}} \E_\bc{D}[\delta_{y,i}] 
		= \frac{1}{2} P_\bc{D}[- \delta_{x,i} < \bar{x}_i < \delta_{x,i}] + P_\bc{D}[\bar{x}_i > \delta_{x,i}] > 0,
	\end{equation}
	and
	\begin{equation}
		\frac{\partial}{\partial^2\delta_{x,i}} \E_\bc{D}[\delta_{y,i}] =\frac{1}{2} (P_\bc{D}[\bar{x}_i = -\delta_{x,i}] - P_\bc{D}[\bar{x}_i = \delta_{x,i}]). 
		\label{eqn:growth_curvature}
	\end{equation}
\end{restatable}

\begin{proof}
	Recall that given an input box with centre $\bar{\vx}$ and radius $\delta_{\vx}$, the output relaxation of a ReLU layer is defined by:
	\begin{equation}
	\bar{y}_i = \begin{cases} 
	0, \quad &\text{if } \bar{x}_i + \delta_{x,i} \leq 0\\
	\frac{\bar{x}_i + \delta_{x,i}}{2}, \quad &\text{elif } \bar{x}_i - \delta_{x,i} \leq 0\\
	\bar{x}_i, \quad &\text{else}
	\end{cases}
	, \qquad
	\delta_{y,i} = \begin{cases} 
	0, \quad &\text{if } \bar{x}_i + \delta_{x,i} \leq 0\\
	\frac{\bar{x}_i + \delta_{x,i}}{2}, \quad &\text{elif } \bar{x}_i - \delta_{x,i} \leq 0\\
	\delta_{x,i}, \quad &\text{else}
	\end{cases}
	\end{equation}
	We thus obtain the expectation
	\begin{align}
	\E_\bc{D}[\delta_{y,i}] &= \int_{-\delta_{x,i}}^{\delta_{x,i}} \frac{\bar{x}_i + \delta_{x,i}}{2} p[\bar{x}_i] d \bar{x}_i + \int_{\delta_{x,i}}^{\infty} \delta_{x,i} p_D(\bar{x}_i) d \bar{x}_i\nonumber \\
	&=	\frac{\delta_{x,i}}{2} P_\bc{D}[- \delta_{x,i} < \bar{x}_i < \delta_{x,i}] + \delta_{x,i} P_\bc{D}[\bar{x}_i > \delta_{x,i}] + \int_{-\delta_{x,i}}^{\delta_{x,i}} \frac{\bar{x}_i}{2} p[\bar{x}_i] d \bar{x}_i,
	\end{align}
	its derivative
	\begin{align}
	\frac{\partial}{\partial\delta_{x,i}} \E_\bc{D}[\delta_{y,i}] 
	=& \frac{1}{2} P_\bc{D}[- \delta_{x,i} < \bar{x}_i < \delta_{x,i}] + \frac{\delta_{x,i}}{2} (P_\bc{D}[\bar{x}_i = -\delta_{x,i}] + P_\bc{D}[\bar{x}_i = \delta_{x,i}]) \nonumber \\
	&+ P_\bc{D}[\bar{x}_i > \delta_{x,i}] - \delta_{x,i} P_\bc{D}[\bar{x}_i = \delta_{x,i}] \nonumber \\
	&+ \frac{\delta_{x,i}}{2} (P_\bc{D}[\bar{x}_i = -\delta_{x,i}] - P_\bc{D}[\bar{x}_i = \delta_{x,i}]) \nonumber \\
	=& \frac{1}{2} P_\bc{D}[- \delta_{x,i} < \bar{x}_i < \delta_{x,i}] + P_\bc{D}[\bar{x}_i > \delta_{x,i}] > 0,
	\end{align}
	and its curvature
	\begin{align}
	\frac{\partial}{\partial^2\delta_{x,i}} \E_\bc{D}[\delta_{y,i}] 
	&= \frac{1}{2} (P_\bc{D}[\bar{x}_i = -\delta_{x,i}] + P_\bc{D}[\bar{x}_i = \delta_{x,i}]) - P_\bc{D}[\bar{x}_i = \delta_{x,i}] \nonumber \\
	&=\frac{1}{2} (P_\bc{D}[\bar{x}_i = -\delta_{x,i}] - P_\bc{D}[\bar{x}_i = \delta_{x,i}]).
	\end{align}
\end{proof}

Now, we can easily proof \cref{thm:growth}, restated below for convenience.

\growth*

\begin{proof}
	We apply \cref{thm:growth_lem} by substituting an asymmetric centre distribution $\bc{D}$, satisfying $P_\bc{D}(\bar{x} = -z) > P_\bc{D}(\bar{x} = z), \; \forall z\in \R^{>0}$ into \cref{eqn:growth_curvature} to obtain:
	\begin{equation*}
		\frac{\partial}{\partial^2\delta_{x,i}} \E_\bc{D}[\delta_{y,i}] = \frac{1}{2} (P_\bc{D}[\bar{x}_i = -\delta_{x,i}] - P_\bc{D}[\bar{x}_i = \delta_{x,i}])>0.
	\end{equation*}
	The theorem follows trivially from the strictly positive curvature.
\end{proof}

\paragraph{Example for Piecewise Uniform Distribution}
Let us assume the centres $\bar{x}\sim D$ are distributed according to:
\begin{equation}
	P_\bc{D}[\bar{x} = z] = \begin{cases}
	a, \quad &\text{if } -l \leq z < 0\\
	b, \quad &\text{elif } 0 < u \leq l\\
	0, \quad &\text{else}
	\end{cases},
	\quad
	l = \frac{1}{a+b},
\end{equation}
where $a$ and $b$. Then we have by \cref{thm:growth_lem}
	\begin{align}
		\E_\bc{D}[\delta_{y}] 
		&=	\frac{\delta_{x}}{2} P_\bc{D}[- \delta_{x} < \bar{x} < \delta_{x}] + \delta_{x} P_\bc{D}[\bar{x} > \delta_{x}] + \int_{-\delta_{x}}^{\delta_{x}} \frac{\bar{x}}{2} p[\bar{x}] d \bar{x}\\
		&= \frac{\delta^2_x}{2} (a+b) +  b \delta_x (l-\delta_x) + \frac{\delta_x^2}{4} (b-a)\\
		&= \delta^2_x \frac{a - b}{4}  +  \delta_x \frac{b}{a+b}.
	\end{align}
We observe quadratic growth for $a>b$ and recover the symmetric special case of $\E_\bc{D}[\delta_{y}] = 0.5 \delta_x$ for $a=b$.

\section{Additional Theoretical Details}
\subsection{Cross-Entropy Loss Formulation}\label{app:ce_deviation}
Below we derive the formulation of the Cross-Entropy (CE) loss used in equation \cref{eq:std_train,eq:rob_loss}. We let $p_i$ be the label probability of class $i$, $q_i$ the predicted probability of class $i$, $t$ the label for sample $\vx$ and $\vy$ the logits predicted by a neural network $\vh$ for this sample.

\begin{align*}
\mathcal{L}_\text{CE} &= -\sum_{i=1}^n p_i(y) \log(q_i(\vy))\\
&= -\sum_{i=1}^n \mathbf{1}_{i=y} \log\left(\frac{\exp(y_i)}{\sum_{j=1}^n \exp(y_j)}\right)\\
&= -\log\left(\frac{\exp(y_t)}{\sum_{j=1}^n \exp(y_j)}\right)\\
&= -\log\left(\frac{\exp(y_t)/\exp(y_t)}{\sum_{j=1}^n \exp(y_j)/\exp(y_t)}\right)\\
&= -\log\left(\frac{1}{\sum_{j=1}^n \exp(y_j-y_t)}\right)\\
&= \log\left(\sum_{j=1}^n \exp(y_j-y_t)\right) - \log(1)\\
&= \log\left(1 + \sum_{\substack{j=1\\ j \neq y}}^n \exp(y_j-y_t)\right)
\end{align*}

%% file: eval_appendix.tex
\section{Additional Experimental Details} \label{app:eval_detail}
In this section, we provide detailed informations on the exact experimental setup.

\paragraph{Datasets} We conduct experiments on the \mnist \citep{lecun2010mnist}, \cifar \citep{krizhevsky2009learning}, and \TIN \citep{Le2015TinyIV} datasets. For \TIN and \cifar we follow \citet{ShiBetterInit2021} and use random horizontal flips and random cropping as data augmentation during training and normalize inputs after applying perturbations. Following prior work \citep{XuS0WCHKLH20,ShiBetterInit2021}, we evaluate  \cifar and \mnist on their test sets and \TIN on its validation set, as test set labels are unavailable. Following \citet{XuS0WCHKLH20} and in contrast to \citet{ShiBetterInit2021}, we train and evaluate \TIN with images cropped to $56\times56$.

\begin{wraptable}[10]{r}{0.35\textwidth}
	\centering
	\vspace{-3mm}
	\begin{minipage}{1.00\linewidth}
		\resizebox{1.00\linewidth}{!}{
	\begin{threeparttable}
		\centering
		\caption{Hyperparameters for the experiments shown in \cref{tab:results}.}
		\label{tab:params}
		\begin{tabular}{lccc}
			\toprule
			Dataset & $\epsilon$ &$\ell_1$& $\lambda$\\ 
			\midrule
			\multirow{2}*{\mnist} & 0.1 &$10^{-5}$ & 0.4\\
			& 0.3&$10^{-6}$& 0.6 \\
			\midrule
			\multirow{2}*{\cifar} & 2/255 &$10^{-6}$& 0.1 \\
			& 8/255 & 0 & 0.7  \\
			\midrule
			\TIN & 1/255 &$10^{-6}$& 0.4 \\
			\bottomrule
		\end{tabular}
	\end{threeparttable}
	}
\end{minipage}
\end{wraptable}
\paragraph{Training Hyperparameters} We mostly follow the hyperparameter choices from \citet{ShiBetterInit2021} including their  weight initialization and warm-up regularization\footnote{For the ReLU warm-up
regularization, the bounds of the small boxes are considered.}, and use ADAM \citep{KingmaB14} with an initial learning rate of $5\times 10^{-4}$, decayed twice with a factor of 0.2. 
For \cifar we train $160$ an $180$ epochs for $\epsilon=2/255$ and $\epsilon=8/255$, respectively, decaying the learning rate after $120$ and $140$ and $140$ and $160$ epochs. For \TIN $\epsilon=1/255$ we use the same settings as for \cifar at $\epsilon=8/255$. For \mnist we train $70$ epochs, decaying the learning rate after $50$ and $60$ epochs. 
We choose a batch size of $128$ for \cifar and \TIN, and $256$ for \mnist.
We use $\ell_1$ regularization with factors according to \cref{tab:params}. 
For all datasets, we perform one epoch of standard training ($\epsilon=0$) before annealing $\epsilon$ from $0$ to its final value over $80$ epochs for \cifar and \TIN and for $20$ epochs for \mnist. 
We use an $n=8$ step PGD attack with an initial step size of $\alpha = 0.5$, decayed with a factor of $0.1$ after the $4$\th and $7$\th step to select the centre of the propagation region. We use a constant \ratiol $\ratios$ with values shown in \cref{tab:params}. For \cifar $\epsilon=2/255$ we use shrinking with $c_s=0.8$ (see below).

\paragraph{ReLU-Transformer with Shrinking}  
Additionally to standard \tool, outlined in \cref{sec:method}, we propose to amplify the \boxd growth rate reduction (see $\cref{sec:theory}$) affected by smaller propagation regions, by adapting the ReLU transformer as follows:
\begin{equation}
\bar{y}_i = \begin{cases} 
0, \quad &\text{if } \bar{x}_i + \delta_{x,i} \leq 0\\
c_s\,\frac{\bar{x}_i + \delta_{x,i}}{2}, \quad &\text{elif } \bar{x}_i - \delta_{x,i} \leq 0\\
\bar{x}_i, \quad &\text{else}
\end{cases}
, \qquad
\delta_{y,i} = \begin{cases} 
0, \quad &\text{if } \bar{x}_i + \delta_{x,i} \leq 0\\
c_s\,\frac{\bar{x}_i + \delta_{x,i}}{2}, \quad &\text{elif } \bar{x}_i - \delta_{x,i} \leq 0\\
\delta_{x,i}, \quad &\text{else}
\end{cases}
.
\end{equation}
We call $c_s$ the shrinking coefficient, as the output radius of unstable ReLUs is shrunken by multiplying it with this factor. We note that we only use these transformers for the \cifar $\epsilon = 2/255$ network discussed in \cref{tab:results}.

\paragraph{Architectures}
Similar to prior work \citep{ShiBetterInit2021}, we consider a 7-layer convolutional architecture, \cnns. The first 5 layers are convolutional layers with filter sizes [64, 64, 128, 128, 128], kernel size 3, strides [1, 1, 2, 1, 1], and padding 1. They are followed by a fully connected layer with 512 hidden units and the final classification. All but the last layers are followed by batch normalization \citep{IoffeS15} and ReLU activations. For the BN layers, we train using the statistics of the unperturbed data similar to \citet{ShiBetterInit2021}. %
During the PGD attack we use the BN layers in evaluation mode. 
We further consider narrower version, \cnnsn which is identical to \cnns expect for using the filter sizes [32, 32, 64, 64, 64] and a fully connected layer with 216 hidden units.

\begin{wraptable}[9]{r}{0.35\textwidth}
	\centering
	\vspace{-5mm}
	\begin{minipage}{1.00\linewidth}
		\resizebox{1.00\linewidth}{!}{
	\begin{threeparttable}
		\renewcommand{\arraystretch}{0.95}
		\caption{\tool training times on a single NVIDIA RTX 2080Ti.
				\vspace{-2mm}}
		\label{tab:time_train}
		\begin{tabular}{lcc}
			\toprule
			Dataset & $\epsilon$ &Time\\ 
			\midrule
			\multirow{2}*{\mnist} & 0.1 & 3h 23 min\\ %
			& 0.3 & 3h 20 min\\ %
			\cmidrule(rl){1-1}
			\multirow{2}*{\cifar} & 2/255 & 7h 6 min\\%425m48.283s \\
			& 8/255 & 7h 20 min\\ %
			\cmidrule(rl){1-1}
			\TIN & 1/255 & 57h 24 min \\%3444m12.114s \\
			\bottomrule
		\end{tabular}
	\end{threeparttable}
	}
\end{minipage}
\end{wraptable}
\paragraph{Hardware and Timings}
We train and certify all networks using single NVIDIA RTX 2080Ti, 3090, Titan RTX, or A6000. Training takes roughly $3$ and $7$ hours for \mnist and \cifar, respectively, with \TIN taking two and a half days on a single NVIDIA RTX 2080Ti. For more Details see \cref{tab:time_train}. Verification with \mnbab takes around $34$h for \mnist, $28$h for \cifar and $2$h for \TIN on a NVIDIA Titan RTX.

\pagebreak
\section{Additional Experimental Results}

\begin{figure}[t]
	\centering
	\begin{subfigure}[t]{0.45\textwidth}
		\centering
		\includegraphics[width=0.95\textwidth]{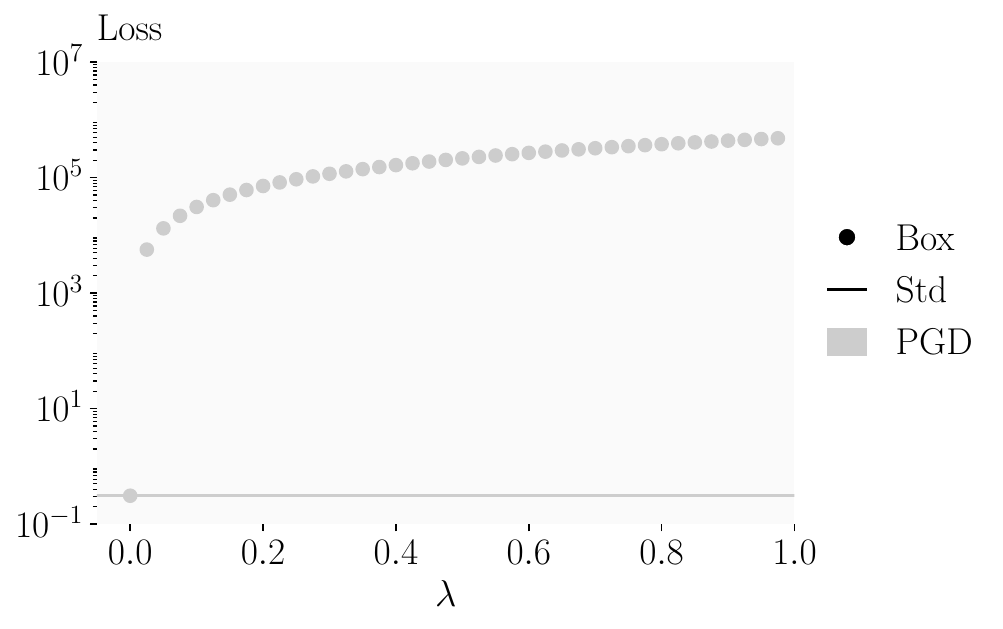}
	\end{subfigure}
	\hfil
	\begin{subfigure}[t]{0.45\textwidth}
		\centering
		\includegraphics[width=0.95\textwidth]{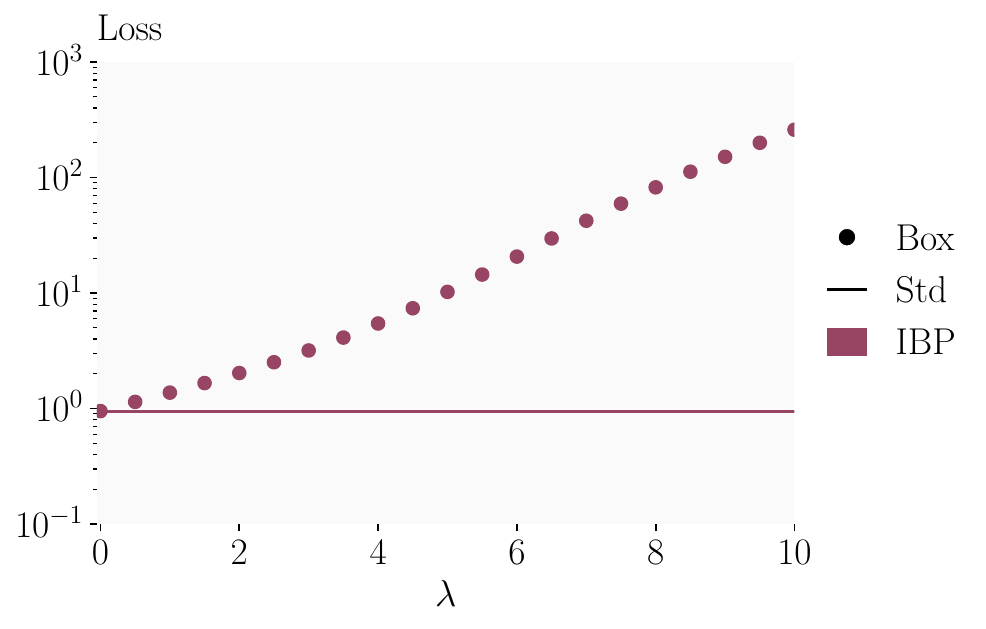}
	\end{subfigure}
	\caption{Standard (Std.) and robust cross-entropy loss, computed with \boxd (Box) bounds for an adversarially (left) and  \ibp (right) trained network over \ratiol $\ratios$. Note the logarithmic y-scale and different axes.}
	\label{fig:error_plot_app}
\end{figure}

\begin{wrapfigure}[15]{r}{0.49\textwidth}
	\centering
	\vspace{-4mm}
	\includegraphics[width=0.99\linewidth]{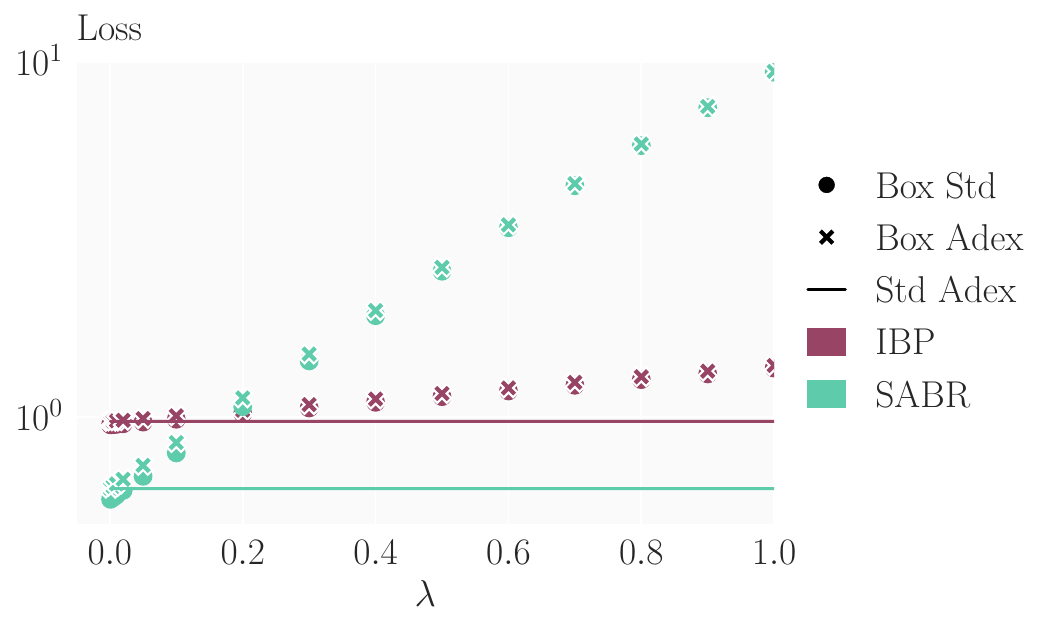}
	\vspace{-7mm}
	\caption{Comparison of the robust cross-entropy losses computed with \boxd (Box) centered around unperturbed and  adversarial examples for an \ibp and  \tool trained network over \ratiol $\ratios$.}
	\label{fig:error_plot_app_adex}
\end{wrapfigure}
\paragraph{Loss Analysis}
In \cref{fig:error_plot_app}, we show the error growth of an adversarially trained (left) and \ibp trained model over increasing \ratiols $\ratios$. We observe that errors grow only slightly super-linear rather than exponential for the adversarially trained network. We trace this back to the large portion of crossing ReLUs (\cref{tab:rel_states}), especially in later layers, leading to the layer-wise growth being only linear. For the \ibp trained model, in contrast, we observe exponential growth across a wide range of propagation region sizes, as the heavy regularization leads to a small portion of active and unstable ReLUs. In \cref{fig:error_plot_app_adex}, we compare errors for \boxd centred around the unperturbed sample (\boxd Std) and around a high loss point computed with an adversarial attack (\boxd Adex). We observe that while the loss is larger around the adversarial centres, especially for small propagation regions, this effect is small compared to the difference between training or certification methods.

\begin{table}[t]
	\centering	
	\begin{adjustbox}{width=\columnwidth,center}
		\begin{threeparttable}
			\caption{Comparison of the standard (Acc.), adversarial (Adv. Acc), and certified (Cert. Acc.) accuracy for different certified training methods on the full \cifar test set. We use \mnbab \citep{FerrariMJV22} to compute all certified and adversarial accuracies.
				\vspace{-2mm}}
			\begin{tabular}{clcccc}
				\toprule
				$\epsilon_\infty$ & Training Method & Source & Acc. [\%] & Adv. Acc. [\%]  & Cert. Acc. [\%] \\
				\midrule
				\multirow{4}*{2/255} &\colt &\citet{BalunovicV20}   &78.42  &	\textbf{66.17} &	61.02 \\
				&\crownibp &\citet{ZhangCXGSLBH20}$^\dagger$                  & 71.27 & 59.58 & 58.19 \\
				&\ibp &\citet{ShiBetterInit2021}                    & - & - & - \\
				&\tool &this work         & \textbf{79.52} & 65.76 & \textbf{62.57}\\
				\cmidrule(rl){1-6}
				\multirow{4}*{8/255}&\colt  &\citet{BalunovicV20}   & 51.69 & 31.81 & 27.60 \\
				&\crownibp & \citet{ZhangCXGSLBH20}$^\dagger$                 & 45.41 & 33.33 & 33.18\\
				&\ibp &\citet{ShiBetterInit2021}                    & 48.94 & 35.43 & \textbf{35.30} \\ 
				&\tool &this work                                   & \textbf{52.00} & \textbf{35.70} & 35.25\\ 
				\bottomrule
			\end{tabular}
			\begin{tablenotes}
				\footnotesize
				\item[] -~$\:\,$No network published.
				\item[] $\dagger$~~Published network does not match reported performance.
			\end{tablenotes}
			\label{tab:mn_bab_results}
		\end{threeparttable}
	\end{adjustbox}
	\vspace{-3mm}
\end{table}

\subsection{Effect of Verification Method on Other Certified Defenses}
In this section we compare different certified defenses when evaluated using the same, precise verifier \mnbab \citep{FerrariMJV22}.
While \colt \citep{BalunovicV20} and \ibpr \citep{PalmaIBPR22} trained networks were verified using similarly expensive and precise verification methods as \mnbab (\milp \citep{TjengXT19} and \bcrown \citep{WangZXLJHK21}, respectively), the \ibp and \crownibp trained networks were originally verified using much less precise \boxd propagation. We compare standard (Acc.), empirical adversarial (Adv. Acc.), and certified (Cert. Acc.) accuracy for \cifar at $\epsilon=2/255$ and $\epsilon=8/255$ in \cref{tab:mn_bab_results}. We omit \ibpr, as neither code nor networks are published. For \crownibp \citep{ZhangCXGSLBH20}, we evaluated both the `best' and last checkpoints of the published networks but observed that the standard accuracy of neither matched the ones reported in the paper. We report the better of the two (`best').

We observe that certified accuracies increase only minimal in most settings, with the exception of \crownibp at $\epsilon=2/255$, where the certified accuracy rises from $54.0\%$ to $58.2\%$. However, there the adversarial accuracy of $59.6\%$ remains significantly below our certified accuracy of $62.6\%$. At $\epsilon=8/255$ the \ibp trained network achieves $35.3\%$ certified accuracy, matching \tool's performance, however at a $3\%$ lower standard accuracy.

\begin{wraptable}[11]{r}{0.46\textwidth}
	\centering
	\vspace{-4.5mm}
	\begin{minipage}{1.0\linewidth}
	\resizebox{1\linewidth}{!}{
	\begin{threeparttable}
		\caption{Ablation of \tool's components with respect to standard (Acc.) and certified (Cert. Acc.) accuracy on the first 1000 samples of the \cifar test set at $\epsilon=2/255$.
			\vspace{-2mm}}
		\begin{tabular}{lcc}
			\toprule
			Training Method & Acc. [\%] & Cert. Acc. [\%] \\
			\midrule
			\tool             & 80.4 & 61.0 \\
			$\quad +$ centred           & 82.5 & 27.4 \\
			$\quad +$ random            & 84.1 & 28.3 \\
			$\quad +$ \fgsm             & 80.8 & 58.2 \\
			$\quad +$ no projection     & 80.0 & 61.9 \\
			$\quad +$ \crownibp         & 80.3 & 56.7 \\
			\bottomrule
		\end{tabular}
		\label{tab:SABR_ablation}
	\end{threeparttable}
	}
	\end{minipage}
\end{wraptable}
\subsection{Ablation \tool}
To assess the different components of \tool, we conduct an ablation study on \cifar with $\epsilon = 2/255$. Beyond the \ratiol $\ratio$, discussed in \cref{sec:experiments} and especially \cref{fig:ablation_lambda}, \tool has two main components: i) the choice of propagation region position and ii) the propagation method.

To analyse the effect of the propagation region's position, we evaluate four methods to compute its center $\vx'$ for $\lambda=0.05$: i) always choose the original input as center (centred: $\vx' = \vx$), ii) choose the center uniformly at random such that the propagation region lies in the original adversarial region (random: $\vx' \sim \bc{U}(\bc{B}^{\epsilon-\tau}(\vx))$), iii) choose the centre with a weak adversarial attack such as \fgsm \citep{GoodfellowSS14} (\fgsm: $\vx' = \vx + (\epsilon-\tau) \sign(\nabla_{\vx} \bc{L}_\text{CE}(\vh_{\bs{\theta}}(\vx),t))$), and iv) choose the centre with a strong adversarial attack over the whole input region $\bc{B}^\epsilon$, without projecting it into $\bc{B}^{\epsilon-\tau}$, allowing it to `stick out' (no projection: $\vx' = \vx^*$).
We observe that choosing centres either at random or as the original input leads to weaker regularization, increasing standard accuracy slightly ($+3.5\%$ and $+2.1\%$, respectively) but significantly reducing certified accuracy ($-32\%$ and $-33\%$, respectively). Using a weaker adversarial attack has a similar but much less pronounced effect, increasing natural accuracy by $0.4\%$ at the cost of a $3.2\%$ reduction in certified accuracy. Permitting propagation regions to protrude from the original input region increases regularization, leading to a slight increase in certified accuracy ($+0.9\%$) at the cost of decreased standard accuracy ($-0.4\%$).

To assess the effect of the propagation method, we compare standard \tool which uses \ibp, yielding an easier optimization problem \citep{JovanovicBBV21}, to \crownibp, which generally yields tighter bounds. We observe that using the less precise \boxd propagation indeed yields better standard and certified accuracy.

\begin{wrapfigure}[16]{r}{0.51\textwidth}
	\centering
	\vspace{-5mm}
	\begin{minipage}{1.1\linewidth}
	\begin{subfigure}[t]{0.385\linewidth}
		\includegraphics[width=0.95\linewidth]{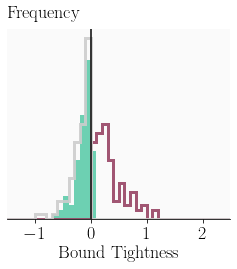}
		\vspace{-2mm}
		\caption{\ibp-trained}
	\end{subfigure}
	\begin{subfigure}[t]{0.55\linewidth}
		\includegraphics[width=0.95\linewidth]{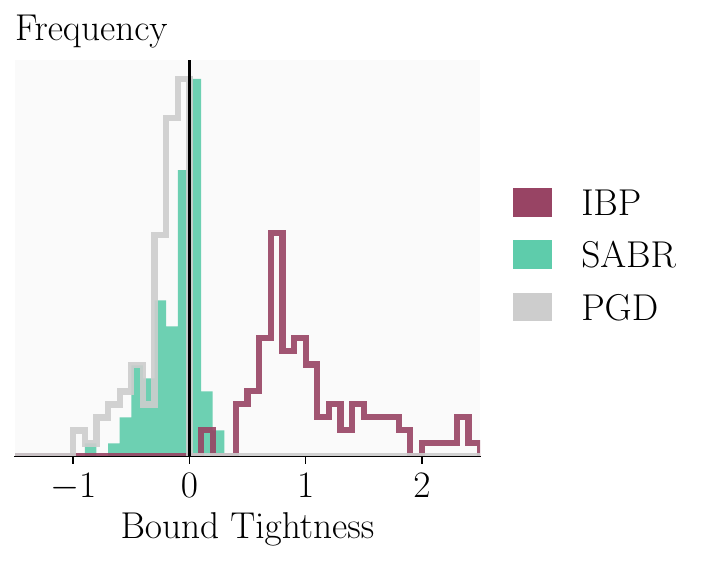}
		\vspace{-2mm}
		\caption{\tool-trained$\qquad\qquad\qquad\;$}
	\end{subfigure}
	\end{minipage}
	\vspace{-3mm}
	\caption{Tightness comparison of bounds computed with \ibp ($\lambda=1$), \tool ($\lambda=0.1$), and an adversarial attack (PGD), relative to the exact \milp \citep{TjengXT19} solution. A bound tightness greater than 0 indicates that the computed approximate bound is over-approximate (e.g. for certification), one below 0 is under-approximate (e.g. an adversarial attack).}
	\label{fig:bound_tightness}
\end{wrapfigure}
\subsection{Bound Tightness}
To support the intuitions discussed in \cref{sec:method}, we compare the tightness of \ibp, $\tool$, and PGD bounds on the worst-case margin loss. To make the computation of the exact worst-case loss ($y^\Delta_\text{\milp}$) via mixed integer linear programming (\milp) and using the encoding from \citet{TjengXT19} tractable, we train a small network (2 convolutional layers and 1 linear layer) with \ibp and \tool ($\lambda=0.1$) on \mnist at $\epsilon = 0.1$. 
We compute bounds on the minimum logit difference $y^\Delta := \min_i y_t - y_i$ for the first 100 test set samples using \tool as during training ($8$-step PGD attack targeting the Cross-Entropy loss), using PGD with a $50$-step margin attack and $3$ restarts per adversarial label, and using \ibp.

In \cref{fig:bound_tightness}, we show histograms of the bound tightness ($y^\Delta_\text{\milp} - y^\Delta_\text{approx}$), where results greater 0 correspond to the bound (PGD, \ibp, or \tool) being \emph{larger} than the actual worst case loss and results smaller 0 correspond to the bound being \emph{smaller}. A sound verification method will thus always yield positive bound tightness and is the more precise the smaller its absolute value. Similarly, adversarial attack based methods will always yield negative bounds, with stronger attacks generally yielding smaller magnitudes.

We observe that, especially for the \tool trained network, \ibp bounds are relatively loose. \tool bounds are generally tighter than both PGD and \ibp bounds (smaller mean magnitude), although clearly not sound (as discussed in \cref{sec:method}). We highlight that the PGD attack used to compute the bound is significantly stronger than the one used for \tool. We conclude that while \tool bounds often do not include the true worst case loss, they represent a better proxy than either \ibp or PGD bounds.

\begin{wraptable}[12]{r}{0.55\textwidth}
	\centering
	\vspace{-4mm}
	\begin{minipage}{1.00\linewidth}
		\resizebox{1.00\linewidth}{!}{
			\begin{threeparttable}
				\caption{Mean and maximum row-wise $\ell_1$-norm of effective weight matrices. Where applicable ($^\ast$), BN layers are merged with the preceding convolutional (Conv) and linear (Lin) layer. 			
					\vspace{-2mm}}
				\label{tab:growth_rates}
				\begin{tabular}{ld{2.2}d{2.2}d{2.2}d{2.2}d{2.2}d{3.2}}
					\toprule
					\multirow{2.5}{*}{Layer} & \multicolumn{2}{c}{PGD} & \multicolumn{2}{c}{\tool} & \multicolumn{2}{c}{IBP} \\
					\cmidrule(rl){2-3} \cmidrule(rl){4-5}				\cmidrule(rl){6-7}
					& \footnotesize mean & \footnotesize max & \footnotesize mean & \footnotesize max & \footnotesize mean & \footnotesize max \\
					\midrule
					Conv 1$^\ast$ & 4.56  & 12.50 &  0.73 &	 2.36 &	 0.64 &   2.74 \\
					Conv 2$^\ast$ & 11.88 & 21.44 &  3.21 & 12.93 &  2.54 &   9.00 \\
					Conv 3$^\ast$ & 11.29 & 18.33 &  4.19 &	17.18 &	17.81 &	 58.97 \\
					Conv 4$^\ast$ & 20.86 &	31.53 &  4.48 &	22.18 &	 6.60 &	 25.75 \\
					Conv 5$^\ast$ & 22.47 & 61.31 &  5.62 & 39.71 & 31.77 & 272.52 \\
					Lin 1$^\ast$  & 30.60 &	98.06 & 19.52 & 83.62 & 32.25 &  65.13 \\
					Lin 2         & 40.07 &	47.29 & 34.32 & 42.06 &	97.20 &	113.65 \\
					\bottomrule
				\end{tabular}
				\label{tab:row_l1_norm}
			\end{threeparttable}
		}
	\end{minipage}
\end{wraptable}
\subsection{Box Growth Rates of Trained Networks}\label{app:growth_rates}
In \cref{tab:row_l1_norm}, we compare the mean and max row-wise $\ell_1$-norm of the effective weight matrices of PGD, \tool and \ibp trained networks for \cifar and $\epsilon=2/255$, depending on the network layer. Where applicable, we combine batch normalization layers with the preceding affine layers. As we show in \cref{sec:theory}, this corresponds to the mean and maximum growth rate $\kappa$ for \boxd of equal side lengths.

We observe that growth rates generally increase with network depth. Interestingly training with \boxd, either in the form of \tool or \ibp significantly reduces growth rates in the early layers, but much less in later layers. For \ibp trained networks, the maximum growth rate in later layers can even exceed that of the PGD trained network.

\subsection{Training Algorithm}
\begin{algorithm}
	\caption{get\_propagation\_region}
	\label{alg:midp_uns_box}
	\begin{algorithmic} 
		\REQUIRE Neural network $\vh$, input $\vx$, label $t$, perturbation radius $\epsilon$, \ratiol $\lambda$, step size $\alpha$, step number $n$
		\ENSURE Center $\vx'$ and radius $\tau$ of propagation region $\bc{B}^{\tau}(\vx')$
		\STATE $(\underline{\vx}, \overline{\vx}) \gets \text{clamp}((\vx-\epsilon, \vx+\epsilon), 0, 1)$ \COMMENT{Get bounds of input region}
		\STATE $\boldsymbol{\tau} \gets \lambda / 2 \cdot (\overline{\vx} - \underline{\vx})$ \COMMENT{Compute propagation region size $\tau$}
		\STATE$\vx^*_0 \gets \text{Uniform}(\underline{\vx}, \overline{\vx})$  \COMMENT{Sample PGD initialization}
		\FOR{$i=0 \dots n-1$} \COMMENT{Do $n$ PGD steps}
		\STATE$\vx^*_{i+1} \gets \vx^*_i + \alpha \cdot \epsilon \cdot \sign(\nabla_{\vx^*_i} \bc{L}_\text{CE}(h(\vx^*_i), t))$
		\STATE$\vx^*_{i+1} \gets \text{clamp}(\vx^*_{i+1}, \underline{\vx}, \overline{\vx})$ 
		\ENDFOR
		\STATE$\vx' \gets \text{clamp}(\vx^*_n, \underline{\vx} + \tau,  \overline{\vx} - \tau)$ \COMMENT{Ensure that $\bc{B}^{\tau}(\vx')$ will lie fully in $\bc{B}^{\epsilon}(\vx)$}
		\RETURN $\vx'´, \tau$
	\end{algorithmic}
\end{algorithm}

\begin{algorithm}
	\caption{\tool Training Epoch}
	\label{alg:train_loop}
	\begin{algorithmic} 
	\REQUIRE Neural network $\vh_{\theta}$, training set $(\mX,\mT)$, perturbation radius $\epsilon$, \ratiol $\lambda$, learning rate $\eta$, $\ell_1$ regularization weight $\ell_1$
	\FOR{$(\vx, t)=(\vx_0, t_0) \dots (\vx_b, t_b)$} \COMMENT{Sample batches $\sim (\mX,\mT)$}
	\STATE$(\vx', \tau) \gets \textrm{get\_propagation\_region}$ \COMMENT{Refer to \cref{alg:midp_uns_box}}
	\STATE$\bc{B}^{\tau}(\vx') \gets \textsc{Box}(\vx', \tau)$ \COMMENT{Get box with midpoint $\vx'$ and radius $\tau$}
	\STATE$\vu_{y^\Delta} \gets \textrm{get\_upper\_bound}(\vh_{\theta}, \bc{B}^{\tau}(\vx'))$ \COMMENT{Get upper bound $\vu_{y^\Delta}$ on logit differences}
	\STATE \COMMENT{based on \ibp}
	\STATE$\textrm{loss} \gets \bc{L}_\text{CE}(\vu_{y^\Delta}, t)$ %
	\STATE$\textrm{loss}_{\ell_1} \gets \ell_1 \cdot \textrm{get\_}\ell_1\textrm{\_norm}(h_{\theta})$
	\STATE$\textrm{loss}_{tot} \gets \textrm{loss} +  \textrm{loss}_{\ell_1}$ 
	\STATE$\theta \gets \theta - \eta \cdot \nabla_{\theta}\textrm{loss}_{tot}$\COMMENT{Update model parameters $\theta$}
	\ENDFOR
	\end{algorithmic}
\end{algorithm}